\documentclass{article}

\usepackage{arxiv}

\usepackage[utf8]{inputenc} 
\usepackage[T1]{fontenc}    
\usepackage{hyperref}       
\usepackage{url}            
\usepackage{booktabs}       
\usepackage{amsfonts}       
\usepackage{nicefrac}       
\usepackage{microtype}      
\usepackage{lipsum}
\usepackage{graphicx}
\usepackage{algorithm,algorithmic}
\usepackage{subfigure}
\usepackage{booktabs} 
\usepackage{amsmath,amsthm,amssymb}
\def\prox{\operatorname{prox}}
\def\R{\mathbb{R}}
\def\N{\mathbb{N}}

\usepackage{tikz}
\usepackage{hyperref}
\usepackage{enumerate} 
\usetikzlibrary{automata,arrows,positioning,calc}
\usepackage[final]{changes}

\newtheorem{lemma}{Lemma}
\newtheorem{theorem}{Theorem}

\newtheorem{rem}{Remark}
\newtheorem{assumption}{Assumption}
\newtheorem{proposition}{Proposition}
\newtheorem{corollary}{Corollary}
\newtheorem{example}{Example}

\graphicspath{ {./images/} }

\title{Delay-adaptive step-sizes for asynchronous learning}

\author{
  Xuyang Wu\\
  KTH Royal Institute of Technology\\
  Stockholm, SE-100 44\\
  \texttt{xuyangw@kth.se}\\
   \And
  Sindri Magn\'{u}sson\\
  Stockholm University\\
  Stockholm, SE-164 07\\
  \texttt{sindri.magnusson@dsv.su.se}\\
  \And
  Hamid Reza Feyzmahdavian\\
  ABB Corporate Research\\
   V\"{a}ster\r{a}s, SE-721 78\\
  \texttt{hamid.feyzmahdavian@se.abb.com}\\
  \And
  Mikael Johansson\\
  KTH Royal Institute of Technology\\
  Stockholm, SE-100 44\\
  \texttt{mikaelj@kth.se}\\
}

\begin{document}
\maketitle
\begin{abstract}
In scalable machine learning systems, model training is often parallelized over multiple nodes that run without tight synchronization. Most analysis results for the related asynchronous algorithms use an upper bound on the information delays in the system to determine learning rates. Not only are such bounds hard to obtain in advance, but they also result in unnecessarily slow convergence. In this paper, we show that it is possible to use learning rates that depend on the actual time-varying delays in the system. We develop general convergence results for delay-adaptive asynchronous iterations and specialize these to proximal incremental gradient descent and block-coordinate descent algorithms. For each of these methods, we demonstrate how delays can be measured on-line, present delay-adaptive step-size policies, and illustrate their theoretical and practical advantages over the state-of-the-art.
\end{abstract}

\section{Introduction}
	\label{submission}
	This paper considers step-sizes that adapt to the true delays in asynchronous algorithms for solving optimization problems in the form
	\begin{equation}\label{eq:generalprob}
		\min_{x\in\mathbb{R}^d}~P(x) = f(x)+R(x),
	\end{equation}
	where $f:\mathbb{R}^d\rightarrow\mathbb{R}$ is a smooth but possibly non-convex loss function and $R:\mathbb{R}^d\rightarrow \mathbb{R}\cup\{+\infty\}$ is a convex nonsmooth function. Here, $R$ is typically a regularizer, promoting desired solution properties such as sparsity, or the indicator function of a closed convex set (the constraint set for $x$).
	
	
	
	
	When either the data dimension (the number of samples defining $f$) or the variable dimension $d$ is large, we may need to distribute the optimization process over multiple compute nodes.
	In a distributed environment, synchronous algorithms such as gradient descent or block coordinate descent, are often inefficient. Since they need to wait for the slowest worker node to complete its task, the system tends to spend a significant time idle  and becomes sensitive to single node failures. This motivates the development of asynchronous algorithms which allow all nodes to run at their maximal capacity without synchronization overhead.

	In the past decade, numerous asynchronous algorithms have been proposed to solve large-scale problems on the form~\eqref{eq:generalprob}. Notable examples include ARock \cite{Peng16}, PIAG \cite{aytekin16,vanli2018global}, Async-BCD \cite{liu2014}, Hogwild! \cite{recht2011hogwild}, AsyFLEXA \cite{cannelli2016}, and DAve-RPG \cite{mishchenko2018delay}, to mention a few. Algorithms that use fixed step-sizes often assume bounded asynchrony and require an upper bound of the worst-case information delay to determine the step-size. However, such an upper bound is usually difficult to obtain in advance, and is a crude model for actual system delays. Indeed, actual latencies may be significantly smaller than the worst case for most nodes, and for most of the time.  This makes the algorithm hard to tune and inefficient to run, since a large worst-case delay leads to a small step-size and a slow iterate convergence.

	\subsection{Algorithms and related work}
	In this paper, we develop general principles and convergence results for asynchronous optimization algorithms that adjust the learning rate on-line to the actual information delays. We then present concrete delay-tracking algorithms and adaptive step-size policies for two specific asynchronous optimization algorithms, PIAG and Async-BCD. These algorithms address two distinct variations of distributed model training: distribution of data over samples (PIAG) and distribution across features (Async-BCD).
	To put our work in context, we review the related literature below.

	\textbf{PIAG}: PIAG solves problem \eqref{eq:generalprob} with aggregated loss $f(x)=\frac{1}{n}\sum_{i=1}^n f^{(i)}(x)$. Here, each $f^{(i)}$ could represent the training loss on sample $i$, on mini-batch $i$ or on the complete data set held by some worker node. The PIAG algorithm is often implemented in a parameter server architecture~\cite{LiM13}, where a master node updates the iterate $x_k$ based on the most recent gradient information from each worker. The new iterate is broadcast to idle workers, who proceed to compute the gradient of the training loss on their local data set, and return the gradient to the master node. Both master and worker nodes operate in an event-driven fashion without any global synchronization.
	
	Early work on PIAG \cite{blatt2007convergent,roux2012stochastic,gurbuzbalaban2017convergence} focused on smooth problems, \emph{i.e.}, let $R\equiv 0$ in~\eqref{eq:generalprob}. Extensions of PIAG that allow for a non-smooth regularizer include  \cite{aytekin16,vanli2018global,Feyzmahdavian21} for convex $f$ and \cite{Deng20,Sun19} for non-convex $f$. In addition, a recent work \cite{wai2020accelerating} compensates for the information delays in PIAG using Hessian information. However, all these papers use an upper bound of the worst-case delay to determine the step-size.

	\textbf{Async-BCD}: Async-BCD splits the whole variable $x$ into multiple blocks $\{x^{(i)}\}_{i=1}^m$ and solves problem \eqref{eq:generalprob} with separable nonsmooth function $R(x) = \sum_{i=1}^m R^{(i)}(x^{(i)})$. The algorithm is usually implemented in a shared memory architecture \cite{Peng16}, where the iterate is stored in shared memory and multiple servers asynchronously and continuously update one block at a time based on the delayed iterates they read from the shared memory.
	
	Existing work on Async-BCD includes \cite{liu2014,liu2015,Davis16,Sun17}, among which
	\cite{Sun17} consider smooth problems  ($R^{(i)}\equiv 0$), 
	\cite{liu2014} requires $R^{(i)}$ to be an indicator function, and \cite{liu2015,Davis16} allow for general convex $R^{(i)}$. In addition, some asynchronous methods use updates that are similar to Async-BCD, such as ARock \cite{Peng16,hannah2018unbounded,Feyzmahdavian21} and AsyFLEXA \cite{cannelli2016}. All these papers except \cite{hannah2018unbounded} consider fixed step-sizes tuned based on a uniform upper bound of the delays, while \cite{hannah2018unbounded} assumes stochastic delays and suggests a step-size that relies on the distribution parameters of the delays (quantities that are often unknown).

	\subsection{Contributions}
	
	\replaced{This paper introduces the concept of delay-adaptive step-sizes for asynchronous optimization algorithms. We demonstrate how information delays can be accurately recorded on-line, introduce a family of dynamic step-size policies that adapt to the true amount of asynchrony in the system, and give a formal proof for convergence under all bounded delays. This eliminates the need to know an upper bound of the delays to set the learning rate and the removes the (typically significant) performance penalty that occurs when this upper bound is larger than the true system delays. We make the following specific contributions:}{
	Getting rid of the uniform delay upper bound in the step-size makes asynchronous algorithms practically implementable. Moreover, by using true delays rather than its uniform upper bound, larger step-size and possible faster convergence may be obtained. We track the true delays in PIAG and Async-BCD and derive delay-adaptive step-sizes that use the true delays. Our contribution can be summarized as}
	\begin{itemize}
		\item We develop simple and practical delay tracking algorithms for PIAG in the parameter server and for Async-BCD in shared memory. 
        \item We derive a novel convergence result that simplifies the analysis of broad classes of asynchronous optimization algorithms, and allows to analyze the effect of a time-varying and delay-dependent learning rate.
        
        \item We demonstrate how a natural extension of the fixed step-sizes proposed for asynchronous optimization to the delay-adaptive setting fails, and suggest a general step-size principle that ensures convergence under all bounded delays, even if their upper bound is unknown. 

		\item Under the step-size principle, we design two delay-adaptive step-size policies that use the true delay. We derive explicit convergence rate guarantees for PIAG and Asynch-BCD under these step-size policies, compare these with the state-of-the art, and identify scenarios where our new step-sizes give large speed-ups. 
	\end{itemize}
	Experiments on a classification problem show that the proposed delay-adaptive step-sizes accelerate the convergence of the two methods compared to the best known fixed step-sizes from the literature. 
	\subsection*{Notation and Preliminaries} We use $\mathbb{N}$ and $\mathbb{N}_0$ to denote the set of natural numbers and the set of natural numbers including zero, respectively. We let $[m] = \{1,\ldots,m\}$ for any $m\in\mathbb{N}$ and define the proximal operator of a function $R:\mathbb{R}^d\rightarrow\mathbb{R}\cup\{+\infty\}$ as
	\begin{equation*}
		\operatorname{prox}_{R}(x) = \underset{y\in\mathbb{R}^d}{\operatorname{\arg\!\min}} ~R(y)+\frac{1}{2}\|y-x\|^2.
	\end{equation*}
	We say a function $f:\R^d\rightarrow \R$ is $L$-smooth if it is differentiable and
	\begin{equation*}
		\|\nabla f(x)-\nabla f(y)\|\le L\|x-y\|,~\forall x,y\in\mathbb{R}^d.
	\end{equation*}
	For $L$-smooth function $f:\R^d\rightarrow \R$ and convex function $R:\R^d\rightarrow \R\cup\{+\infty\}$, we say that $P(x)=f(x)+R(x)$ satisfies the proximal PL condition \cite{karimi16} with some $\sigma>0$ if
	\begin{equation}\label{eq:proximalPLcond}
		\sigma(P(x)-P^\star)\le -L\hat{P}(x),~\forall x\in\operatorname{dom}(P),
	\end{equation}
	where $\hat{P}(x)=\min_{y\in\mathbb{R}^d}~\{\langle \nabla f(x), y-x\rangle+\frac{L}{2}\|y-x\|^2+R(y)-R(x)\}$ and $P^\star=\min_{x\in\mathbb{R}^d} P(x)$.
	\section{Algorithms with delay-tracking}
	
	In this section, we first introduce the PIAG and Async-BCD algorithms and demonstrate how they can record actual system delays with almost no overhead. The key to this observation is that delays in asynchronous algorithms are typically not measured in physical time, but rather in the number of write events that have occurred since the model parameters that are used in the update were computed~(see, \emph{e.g.},~\cite{leblond:18}). Hence, in the parameter server architecture and the shared memory systems, delays can often be computed accurately without any intricate time synchronization between distributed nodes. We then demonstrate how the natural extension of the state-of-the-art step-size rules for worst-case delays fails to extend to time-varying delays.

	
	\subsection{PIAG in a parameter server architecture}\label{ssec:PIAG}

	PIAG solves problem \eqref{eq:generalprob} with aggregated loss $f(x)=\frac{1}{n}\sum_{i=1}^n f^{(i)}(x)$ and takes the following form:
	\begin{align}
		& g_k = \frac{1}{n}\sum_{i=1}^n \nabla f^{(i)}(x_{k-\tau_k^{(i)}}),\label{eq:gkdef}\\
		& x_{k+1} = \prox_{\gamma_k R}(x_k-\gamma_k g_k)\label{eq:update},
	\end{align}
	where $\tau_k^{(i)}\in [0,k]$ is the delay of $\nabla f^{(i)}$ at the $k$th iteration.

	\textbf{Parameter server}: PIAG is usually implemented in a parameter server framework \cite{LiM13} with one master and $n$ workers, each one capable of computing (stochastic, mini-batch, or full) gradients of a specific $f^{(i)}$. The master maintains the most recent iterate $x_k$ and the most recently received gradients $g^{(i)}=\nabla f^{(i)}(x_{k-\tau_k^{(i)}})$ from each worker. Once the master receives new gradients, it revises the corresponding $g^{(i)}$, updates the iterate, and pushes the new parameters back to idle workers. A detailed implementation of PIAG \eqref{eq:gkdef} -- \eqref{eq:update} in the parameter server setting without delay-tracking is presented in \cite{aytekin16}.
	
	
	\textbf{Delay-tracking}: To compute the delays $\tau_k^{(i)}$, the PIAG algorithm needs to know the iteration index of the model parameters used to compute each $g^{(i)}$. In Algorithm~\ref{alg:PIAG}, we maintain this information using a simple time-stamping procedure. Specifically, in iteration $l$, the master pushes the tuple $(x_{l}, l)$ to idle workers. Workers return $(\nabla f^{(i)}(x_{l}), l)$ which the master stores as $g^{(i)}\leftarrow \nabla f^{(i)}(x_{l})$, $s^{(i)}\leftarrow l$. At any iteration $k\geq l$, the delay $\tau_k^{(i)}$ is then given by $k-s^{(i)}$. 

	The tracking scheme in Algorithm \ref{alg:PIAG} can be extended to other approaches that can also be implemented in the parameter server setting, such as Asynchronous SGD \cite{dean2012large,zhang2020taming}.

	\begin{algorithm}[tb]
		\caption{PIAG with delay-tracking}
		\label{alg:PIAG}
		\begin{algorithmic}[1]
			\STATE {\bfseries Input}: initial iterate $x_0$, number of iteration $k_{\max}\in\N$.
			\STATE {\bfseries Initialization:}
			\STATE The master sets $k\leftarrow0$, $g^{(i)}\leftarrow\nabla f^{(i)}(x_0)$ $\forall i\in [n]$, and $g_0\leftarrow\frac{1}{n}\sum_{i=1}^n \nabla f^{(i)}(x_0)$.
			\WHILE{$k\le  k_{\max}$: each worker $i\in [n]$ \emph{asynchronously} and \emph{continuously}}
			\STATE receive $(x_k, k)$ from the master.
			\STATE compute $\nabla f^{(i)}(x_k)$.
			\STATE send $(\nabla f^{(i)}(x_k), k)$ to the master.
			\ENDWHILE
			\WHILE{$k\le  k_{\max}$: the master}
			\STATE Wait until a set ${\mathcal R}$ of workers return.
			\FOR{all $w\in {\mathcal R}$}
			\STATE update $(g^{(w)}, s^{(w)}) \leftarrow  (\nabla f^{(w)}(x_l), l)$.
			\ENDFOR
			\STATE set $g_k \leftarrow \frac{1}{n}\sum_{i=1}^n g^{(i)}$.
			\STATE calculate the delay $\tau_k^{(i)}=k-s^{(i)}$ $\forall i\in [n]$.
			\STATE determine the step-size $\gamma_k$ based on $\tau_k^{(i)}$ $\forall i\in [n]$.
			\STATE update $x_{k+1}\leftarrow\operatorname{prox}_{\gamma_kR}(x_k-\gamma_k g_k)$.
			\STATE set $k\leftarrow k+1$.
			\FOR{all $w\in {\mathcal R}$}
			\STATE send $(x_{k}, k)$ to worker $w$.
			\ENDFOR
			\ENDWHILE
		\end{algorithmic}
	\end{algorithm}

	\subsection{Async-BCD in the shared memory setting}\label{ssec:ARock}
	
	Block-coordinate descent, BCD,~\cite{hong2017iteration} can be a powerful alternative for solving (\ref{eq:generalprob}) when the regularizer is separable. 
	Assume that $R(x) =  \sum_{i=1}^m R^{(i)}(x^{(i)})$ where $x = (x^{(1)}, \ldots, x^{(m)})$, $x^{(i)}\in\R^{d_i}$ and $\sum_{i=1}^m d_i = d$. At each iteration of BCD, a random $j\in [m]$ is drawn and
%
%
	\begin{equation*}
		\!x_{k+1}^{(j)} \!\!=\! \prox_{\gamma_k R^{(j)}}\!(x_k^{(j)} \!-\gamma_k\nabla_j f(x_k)).
	\end{equation*}
	Async-BCD parallelizes this update over $n$ workers in a shared memory architecture~\cite{Peng16}. Workers operate without synchronization, repeatedly read the current iterate from shared memory, and update a randomly chosen block.
	More specifically, suppose that at time $k$, worker $i_k$ updates the $j$th block $x_k^{(j)}$ based on the partial gradient $\nabla_j f$ at $\hat{x}_k$, where $\hat{x}_k$ is what the server $i_k$ read from the shared memory. Then, the $k$th update is
	\begin{equation}\label{eq:asyncor}
		\!x_{k+1}^{(j)} \!\!=\! \prox_{\gamma_k R^{(j)}}\!(x_{k}^{(j)} \!-\gamma_k\nabla_j f(\hat{x}_k)).
	\end{equation}
	A specific aspect of Async-BCD is that while $i_k$ reads from the shared memory, other workers may be in the process of writing. Hence, $\hat{x}_k$ itself may never have existed in the shared memory. This phenomenon is known as \emph{inconsistent read}~\cite{liu2015}. 
%
    However, if we assume that each (block) write is atomic, then we can express $x_k$ as
	\begin{equation}\label{eq:inconsistentreadhatx}
		x_k = \hat{x}_k + \sum_{j\in J_k} (x_{j+1}-x_j).
	\end{equation}
	where $J_k\subseteq \{0, 1, \dots, k\}$. The sum represents all updates that have occurred since $i_k$ began reading $\hat{x}_k$ until the block update is written back to memory.
	We call $\tau_k=k-\min\{j:j\in J_k\}$ the delay of $\hat{x}_k$ at iteration $k$.
	\textbf{Delay-tracking:} To track the delays in Async-BCD, workers need to record the value of the iterate counter when they begin reading from shared memory, and then again when they begin writing back their result. When worker $i$ begins to read $x$ from the shared memory in Algorithm~\ref{alg:Coor}, 
	it stores the current value of the iterate counter into a local variable $s^{(i)}$. In this way, it can compute the delay $\tau_k=k-s^{(i)}$ when it is time to write back the result at iteration $k$.  
%
We assume that during steps 5-9, worker $i_k$ is the only one that updates the shared memory. This is a little more restrictive than standard Async-BCD that only assumes that the write operation on Line~8 is atomic, but is needed to make sure that $\gamma_k$ calculated in step~6 is used in~\eqref{eq:asyncor} to update $x_{k+1}^{(j)}$.
	
	The tracking technique in Algorithm \ref{alg:Coor} is applicable to many other methods for shared memory systems, such as ARock \cite{Peng16}, Hogwild! \cite{recht2011hogwild}, and AsyFLEXA \cite{cannelli2016}. 

	\begin{algorithm}[tb]
		\caption{Async-BCD with delay tracking}
		\label{alg:Coor}
		\begin{algorithmic}[1]
			\STATE {\bfseries Setup:} initial iterate $x_0$, number of iteration $k_{\max}\in\N$.
			\WHILE{$k\le  k_{\max}$: each worker $i\in [n]$ \emph{asynchronously} and \emph{continuously}}
			\STATE sample $j\in [m]$ uniformly at random.
			\STATE compute $\nabla_j f(\hat{x}_k)$ based on $\hat{x}_k$ read at time $s^{(i)}$.
			\STATE calculate $\tau_k=k-s^{(i)}$.
			\STATE determine the step-size $\gamma_k$.
			\STATE compute $x_{k+1}^{(j)}$ by \eqref{eq:asyncor}.
			\STATE write on the shared memory.
			\STATE set $k\leftarrow k+1$.
			\STATE set $s^{(i)}=k$ and read $x_k$ from the shared memory.
			\ENDWHILE
		\end{algorithmic}
	\end{algorithm}

	\subsection{Intuitive extension of a fixed step-size fails}
	
	Several of the least conservative results for PIAG~\cite{Sun19,Deng20,Feyzmahdavian21} and Async-BCD~\cite{Davis16,Sun17} use step-sizes on the form $\gamma_k = \frac{c}{\tau+b}$ where $b$ and $c$ are positive constants (independent of the delays) and $\tau$ is the maximal delay. A natural candidate for a delay-adaptive step-size would be one where the upper delay bound is replaced by the true system delay, \emph{i.e.} 
	\begin{equation}\label{eq:simpleextensionfix}
		\gamma_k = \frac{c}{\tau_k+b}.
	\end{equation}
	However, as the next example demonstrates, this step-size can lead to divergence even for simple problems.
	
	\begin{example}\label{exam:counterexample}
		Consider problem \eqref{eq:generalprob} with $n=d=1$, $f(x)=\frac{1}{2}x^2$, and $R(x)=0$. Suppose that $\tau_k=k \operatorname{mod} T$ for all $k\in\mathbb{N}_0$ for some $T >\!b(e^{2/c}-1)$. Then, the delays are bounded by $T\!-\!1$ and both PIAG and Async-BCD update as
		\begin{equation*}
			x_{k+1} = x_k - \gamma_k \nabla f(x_{k-\tau_k}) = x_k - \gamma_k x_{T\cdot \lfloor k/T \rfloor},
		\end{equation*}
		so that $x_{(k+1)T}=(1-\sum_{t=0}^{T-1} \gamma_{kT+t})x_{kT}$. Then, $\{x_{kT}\}$ diverges if $\sum_{t=0}^{T-1} \gamma_{kT+t}>2$, which is indeed true by \eqref{eq:simpleextensionfix}:
		\begin{equation*}
			\begin{split}
				\sum_{t=0}^{T-1} \gamma_{kT+t} \ge \sum_{t=0}^{T-1} \frac{c}{t+b} \ge \int_b^{T+b} \frac{c}{s}~ds = c\ln \frac{T+b}{b}>2.
			\end{split}
		\end{equation*}
	\end{example}
	
	\added{However, as we will demonstrate next, convergence can be guaranteed under a slightly more advanced step-size policy. }
    	
	\section{Delay-adaptive step-size}\label{sec:stepsize}
	
\replaced{In this section, we prove that both PIAG and Async-BCD converge under step-size policies that satisfy}{
 In this section, we illustrate the main contribution of the paper: how can we adapt step-sizes to the delay in run-time?  In particular, we prove that both PIAG and Async-BCD converge with any step-size $\sum_{t=0}^{\infty} \gamma_t = \infty$ if}
 	\begin{equation}\label{eq:commonstepsizerule}
		0\le \gamma_k \le \max(0,\gamma'- \sum_{t=k-\tau_k}^{k-1} \gamma_t)
	\end{equation}
	provided that also $\sum_{t=0}^{\infty} \gamma_t = +\infty$. Here, the constant $\gamma^{\prime}$ only depends on loss function properties, and there is no need to know the maximal value of the system delay to tune, run, or certify the system. To the best of our knowledge, this is the first such result in the literature.
	
	The convergence analysis is based on a new sequence result for asynchronous iterations, that could be applicable to many problems beyond the scope of this paper. We provide convergence results for PIAG and Async-BCD for several classes of problems in sections~\ref{ssec:PIAGconvergence} and~\ref{ssec:asyncbcdconverge}, respectively. In Section~\ref{ssec:adaptivestep}, we introduce a few specific step-size policies that satisfy the general principle (\ref{eq:commonstepsizerule}) and demonstrate how they extend and improve existing fixed step-sizes both in theory and practice. 

	\subsection{Novel sequence result for delay-adaptive sequences}
	\label{ssec:sequence}


	
	
    Lyapunov theory, and related sequence results, are the basis for the convergence analysis of many optimization algorithms~\cite{polyak1987introduction}. Asynchronous algorithms are no different~\cite{aytekin16,Peng16,Davis16,bertsekas1989convergence}. Several convergence results for asynchronous algorithms are unified and generalized in a recent work \cite{Feyzmahdavian21}.
	%
However, previous work has focused on only scenarios where the maximum delay is known, and existing results cannot be used to analyse delay-adaptive step-sizes like \eqref{eq:commonstepsizerule}. The following theorem generalizes these results to allow adaption to the actual observed delay.  
	\begin{theorem}\label{theo:sequence}
		Suppose that the non-negative sequences $\{V_k\}$, $\{X_k\}$, $\{W_k\}$, $\{p_k\}$, $\{r_k\}$, and $\{q_k\}$ satisfy
		\begin{equation}\label{eq:asynseq}
			X_{k+1}+V_{k+1} \le q_kV_k+p_k\sum_{\ell = k-\tau_k}^{k-1} W_\ell - r_kW_k
		\end{equation}
		for all $k\in\N_0$, where $q_k\in (0,1]$ and $\tau_k\in [0,k]$. Let $Q_k = \Pi_{j=0}^{k-1} q_j$, $k\in\mathbb{N}_0$. If for all $k\in\mathbb{N}_0$, either $p_k=0$ or
		\begin{equation}\label{eq:generalprincipal}
			\frac{p_k}{Q_{k+1}}\le \frac{r_\ell}{Q_{\ell+1}}-\sum_{t=\ell+1}^{k-1} \frac{p_t}{Q_{t+1}},\quad \forall \ell \in [k-\tau_k, k],
		\end{equation}
		then
		\begin{equation}\label{eq:Vconverge}
			V_k\le Q_kV_0,~\forall k\in\mathbb{N}
		\end{equation}
		and
		\begin{equation}\label{eq:summableX}
			\sum_{k=1}^\infty \frac{X_k}{Q_k}\le V_0.
		\end{equation}
	\end{theorem}
	\begin{proof}
		See Appendix \ref{proof:theoseq}.
	\end{proof}
	The theorem is a tool for establishing the convergence and convergence rate of $X_k$ and $V_k$.
	The condition \eqref{eq:asynseq} is quite general, so the result\deleted{s} may be useful \replaced{for}{to study} many \deleted{other} methods beyond PIAG and Async-BCD \replaced{that we focus on in this paper.}{, though PIAG and Async-BCD are our main focus. }
	
	\replaced{The theorem can be used to establish a linear convergence rate of algorithms. In particular, if $q_k\le q$ for all $k\in\mathbb{N}_0$ and some $q\in(0,1)$ then~\eqref{eq:Vconverge}--\eqref{eq:summableX} imply the linear rates
	\begin{equation*}
		V_k\le q^kV_0,~~X_k\le q^kV_0.
	\end{equation*}
}{The theorem can be used to establish a linear convergence rate of algorithms. In particular, if $q_k\le q$ for all $k\in\mathbb{N}_0$ and some $q\in(0,1)$ then~\eqref{eq:Vconverge} implies the linear rate
	\begin{equation}\label{eq:Vkfixedq}
		V_k\le q^kV_0
	\end{equation}
	and \eqref{eq:summableX} ensures a linear convergence of $X_k$ since
	\begin{equation}\label{eq:Xkfixedq}
		\lim_{k\rightarrow\infty} \frac{X_k}{q^k} = 0.
	\end{equation}}
	

	If we can only say that $q_k\leq 1$, 
	then \eqref{eq:summableX} yields that
	\begin{equation*}
		\sum_{k=1}^\infty X_{k} < +\infty,
	\end{equation*}
	from which we conclude that 
	$\lim_{k\rightarrow \infty} X_k = 0$. 

\subsection{Convergence of PIAG under principle \eqref{eq:commonstepsizerule}}\label{ssec:PIAGconvergence}

\replaced{With the help of Theorem~\ref{theo:sequence}, we are able to establish the following convergence guarantees for PIAG under the general step-size principle (\ref{eq:commonstepsizerule}). 
}{
	
We now illustrate how Theorem \ref{theo:sequence} is used to establish the convergence of PIAG with adaptive step-sizes for non-convex, convex, and {\color{red}PL?} strongly convex optimization problems.
}
The main proof idea is to show that some quantities generated by PIAG satisfy the equation \eqref{eq:asynseq} when \eqref{eq:commonstepsizerule} holds (see Lemma \ref{lemma:PIAGsequence} in the Appendix).
	\begin{theorem}\label{theo:nonconvexPIAG}
	Suppose that each  $f^{(i)}$ is $L_i$-smooth, $R$ is convex and closed, and $P^\star:=\min_x P(x)>-\infty$. Define $L=\sqrt{({1}/{n})\sum_{i=1}^n L_i^2}.$
    \replaced{Let $\{x_k\}$ be generated by the PIAG algorithm with a step-size sequence $\{\gamma_t\}$ that satisfies (\ref{eq:commonstepsizerule})}{ 
	If we apply the adaptive step-size principle~\eqref{eq:commonstepsizerule}}  with $\gamma' = {h}/{L}$ for some $h\in (0,1)$\replaced{. Then,}{then, following holds:}
		\begin{enumerate}[(1)]
			\item \deleted{\textbf{Non-convex:}} For each $k\in\mathbb{N}_0$, there exists $\xi_k\in \partial R(x_k)$ such that
			\begin{equation*}
				\!\!\!\sum_{k=1}^\infty\!\gamma_{k-1}\|\nabla f(x_k)\!+\!\xi_k\|^2\!\le\!\frac{2(h^2\!-\!h\!+\!1)(P(x_0)\!-\!P^\star)}{1-h}.
			\end{equation*}
			\item \deleted{\textbf{Convex:}} If each $f^{(i)}$ is convex, then
			\begin{equation*}
				P(x_k)-P^\star\le \frac{P(x_0)-P^\star+\frac{1}{2a_0}\|x_0-x^\star\|^2}{1+\frac{1}{a_0}\sum_{t=0}^{k-1}\gamma_t},
			\end{equation*}
			where $a_0=\frac{h(h+1)}{L(1-h)}$.
			\item \deleted{\textbf{PL:}} If $P$ satisfies the proximal PL-condition \eqref{eq:proximalPLcond}, then
			\begin{equation*}
				P(x_k)\!-\!P(x^\star)\!\le\! e^{-\frac{3c\sigma(1-\tilde{h})}{4(\tilde{h}^2-\tilde{h}+1)}\sum_{t=0}^{k-1}\gamma_t}(P(x_0)\!-\!P^\star),
			\end{equation*}
			where $\tilde{h} = \frac{1+h}{2}$ and $c = \min\left(1,\frac{1-h}{2h}\frac{L}{\sigma}\right)$.
		\end{enumerate}
	\end{theorem}
	\begin{proof}
		See Appendix \ref{proof:theoPIAG}.
	\end{proof}
	
	\added{The three cases roughly represent non-convex (1), convex (2) and strongly convex (3) objective functions, but note that the PL condition is less restrictive that strong convexity and can also be satisfied by some non-convex functions.}
	
	To get explicit convergence rates, we need to specialize the results to a specific step-size policy; we will do this in Section~\ref{ssec:adaptivestep}. Still, we can already now notice that the sum of the step-sizes, $\sum_{t=0}^{k-1}\gamma_t$ that dictates the convergence speed. This is immediate in case (2) and (3), but is also true in case (1), since the non-convex result also implies that 
	\begin{equation*}
		\min_{1\le t\le k}\!\|\nabla f(x_t)\!+\xi_t\|^2\!\le\! \frac{2(h^2\!-\!h\!+\!1)(P(x_0)\!-\!P^\star)}{(1-h)\sum_{t=0}^{k-1} \gamma_t}.
	\end{equation*}
	\deleted{
	The theorem establishes the convergence of PIAG under \eqref{eq:commonstepsizerule}. The convergence speed in \eqref{eq:runningbestPIAG}  and parts (2) -- (3) is dictated by step-size integral $\sum_{t=0}^{k-1} \gamma_t$. In particular, larger step-size integral results in smaller error bound and faster convergence. 
	We illustrate, further, in Section~\ref{ssec:adaptivestep} how to get convergence rates for particular adaptive step-sizes.
	}

	\deleted{Note, also, that the PL condition in part (3) includes some non-convex problems. }

	\subsection{Convergence of Async-BCD under principle \eqref{eq:commonstepsizerule}}\label{ssec:asyncbcdconverge}


\replaced{Next, we}{
We now illustrate how}  \deleted{is used} establish the convergence of Async-BCD with adaptive step-sizes for non-convex optimization problems. The following assumption \replaced{is}{will be} useful.
\begin{assumption} \label{asm:sc}
    \deleted{The} $f$ is differentiable and there exists $\hat{L}>0$ such that  for all $i,j\in[m]$ and $x\in \mathbb{R}^d$
    the following holds\footnote{$U_j:\mathbb{R}^{d_j}\rightarrow\mathbb{R}^d$ maps $x^{(j)}\in\mathbb{R}^{d_j}$ into a $d$-dimensional vector where the $j$th block is $x^{(j)}$ and other blocks are $0$.}
	\begin{equation*}
		\|\nabla_i f(x+U_jh_j) - \nabla_i f(x)\|\le \hat{L}\|h_j\|,~\forall h_j\in\mathbb{R}^{d_j}.
	\end{equation*}
\end{assumption}
The assumption \replaced{implies that $f$ is $L$-smooth}{indicates $L$-smoothness of $f$} for some $L\in [\hat{L}, m\hat{L}]$. We consider the block-wise constant $\hat{L}$ rather than $L$ because the former one is smaller, which in turn leads to larger step-size and faster convergence.
\deleted{We are now ready to present the main result. }

By showing that some quantities generated by Async-BCD satisfy the equation \eqref{eq:asynseq} in Theorem \ref{theo:sequence} (see Lemma \ref{lemma:ARock} in the Appendix), we derive the following theorem.
	\begin{theorem}\label{theo:coor}
	Suppose that each $R^{(i)}$ is convex  and closed,  $P^\star:=\min_x P(x)>-\infty$, and Assumption~\ref{asm:sc} holds. Let $\{x_k\}$ be generated by the Async-BCD algorithm with a step-size sequence $\{\gamma_t\}$ that satisfies (\ref{eq:commonstepsizerule}) with $\gamma' = {h}/{\hat{L}}$ for some $h\in (0,1)$. Then,
		\begin{equation*}
			\sum_{k=0}^\infty \gamma_kE[\|\tilde{\nabla} P(x_k)\|^2]\le \frac{4m(P(x_0)-P^\star)}{1-h},
		\end{equation*}
		where $\tilde{\nabla} P(x_k)=\hat{L}(\operatorname{prox}_{\frac{1}{\hat{L}} R}(x_k-\frac{1}{\hat{L}} \nabla f(x_k))-x_k)$.
	\end{theorem}
	\begin{proof}
		See Appendix \ref{proof:theocoor}.
	\end{proof}
	The theorem establishes the convergence of Async-BCD under adaptive step-sizes.
	Note that $\tilde{\nabla} P(x)=\mathbf{0}$ if and only if $\mathbf{0}\in \partial P(x)$, i.e., $x$ is a stationary point of problem \eqref{eq:generalprob}. Moreover, Theorem \ref{theo:coor} implies
	\begin{equation*}
		\min_{1\le t\le k}\|\tilde{\nabla} P(x_t)\|^2\le \frac{4m(P(x_0)-P^\star)}{(1-h)\sum_{t=0}^{k} \gamma_t}.
	\end{equation*}
	Similar to PIAG, a larger step-size integral leads to smaller error bound in the above equation, which intuitively implies faster convergence of Async-BCD.
	
	\subsection{Delay-adaptive step-size satisfying \eqref{eq:commonstepsizerule}}\label{ssec:adaptivestep}
	
	By the analysis in Section \ref{ssec:PIAGconvergence}--\ref{ssec:asyncbcdconverge}, all step-sizes satisfying $\sum_{t=0}^\infty \gamma_t=+\infty$ and the principle \eqref{eq:commonstepsizerule} guarantee convergence of PIAG and Async-BCD. In this section, we make these results more concrete for two specific adaptive step-size policies that both satisfy 
	\eqref{eq:commonstepsizerule}:
    \begin{description}
        \item[Adaptive 1:] for some $\alpha\in (0,1]$,
    	\begin{equation}\label{eq:PIAGadapt}
		    \gamma_k = \alpha\max\{\gamma' - \sum_{t=k-\tau_k}^{k-1} \gamma_t,0\}.
	    \end{equation}
	    \item[Adaptive 2:]
	    \begin{equation}\label{eq:adapt2}
		    \gamma_k = 
		    \begin{cases}
		      \frac{\gamma^{\prime}}{\tau_k+1}, & \frac{\gamma^{\prime}}{\tau_k+1} \le \gamma'-\sum_{t=k-\tau_k}^{k-1} \gamma_t,\\
		      0, & \text{otherwise}.
		    \end{cases}
	    \end{equation}
    \end{description}

    \replaced{In contrast to existing step-size proposals for asynchronous optimization algorithms, these step-size policies use the actual system delays and do not depend on a (potentially large) upper bound of the maximal delay. When the system operates with small or no delays, these step-sizes approach $\gamma^{\prime}$,
    and if the delays grow large, the step-sizes will be automatically reduced (and occasional updates may be skipped) to guarantee convergence. The performance improvements of these policies over fixed step-size policies depend on the precise nature of the actual delays.}{
	Existing fixed step-sizes in asynchronous algorithms often rely on an usually unknown upper bound of the worst-case delay, while the two delay-adaptive step-sizes \eqref{eq:PIAGadapt}--\eqref{eq:adapt2} only use true delays which are already available in Algorithms \ref{alg:PIAG}, \ref{alg:Coor}. As a result, the two step-sizes are easier to implement compare to fixed step-sizes in the literature. }
	
    \replaced{
    We begin by proving that the two adaptive step-size policies are no worse than the state-of-the-art step-sizes (that require knowledge of the maximal delay). As shown in sections~\ref{ssec:PIAGconvergence}--\ref{ssec:asyncbcdconverge}, the convergence speed depends on the sum of step-sizes. Our first observation is therefore the following.
	}{
	According to \eqref{eq:runningbestPIAG}, \eqref{eq:runningbestcoor}, and Theorems \ref{theo:nonconvexPIAG}--\ref{theo:coor}, larger step-size integral potentially leads to faster convergence of PIAG and Async-BCD. An unexpected result is that in terms of the step-size integral $\sum_{t=0}^k \gamma_t$, the two delay-adaptive step-sizes are at least not worse than two fixed step-sizes that require upper bounds of the worst-case delay (See Proposition \ref{prop:stepsizeintegralbound}). {\color{red}In addition, in case $\tau_t=\tau$ for some $t\in [0,k]$, the largest possible step-size integral is $\sum_{t=0}^k \gamma_t=k\gamma'$, which can be attained by Adaptive 1 with $\alpha=1$ and Adaptive 2 for the burst delay, i.e., $\tau_t=\tau$ at one epoch in $[0,k]$ and $\tau_t=0$ otherwise.}
	
	{\color{red} proof: suppose that $\tau_{k'}=\tau$ for some $k'\in [0,k]$. If $\gamma_{k'}=0$, then because each $\gamma_t\le \gamma'$, we have $\sum_{t=0}^k \gamma_t \le k\gamma'$; Otherwise, by \eqref{eq:commonstepsizerule}, $\sum_{t=k'-\tau}^{k'} \gamma_t\le \gamma'$. In addition, $\gamma_t\le \gamma'$ for all $t\in[0,k]$. We have $\sum_{t=0}^k \gamma_t\le (k-\tau+1)\gamma'\le k\gamma'$.
	
	For Adaptive 1 with $\alpha=1$ and Adaptive 2, in case of burst delay, they have $k$ step-sizes being $\gamma'$ in $[0,k]$, and one step-size being $0$, so that they attain the upper bound $k\gamma'$.
	}

}
	\begin{proposition}\label{prop:stepsizeintegralbound}
		Suppose that $\tau_k\le \tau$ for all $k\in\mathbb{N}_0$. Under the step-size policy~\eqref{eq:PIAGadapt}, it holds that
		\begin{equation}\label{eq:stepsizesummationPIAG}
			\sum_{t=0}^k\gamma_t\ge (k+1)\cdot\frac{\alpha\gamma'}{\tau+1},
		\end{equation}
		while the step-size policy~\eqref{eq:adapt2} guarantees that
		\begin{equation}\label{eq:stepsizeintegral2}
			\sum_{t=0}^k \gamma_t\ge (k+1)\cdot\frac{\tau \gamma'}{(\tau+1)^2}.
		\end{equation}
	\end{proposition}
	\begin{proof}
		See Appendix \ref{sec:proofprop}.
	\end{proof}
	\replaced{The lower bounds in Proposition~\ref{prop:stepsizeintegralbound} are comparable with $k+1$ applications of the state-of-the art fixed step-sizes for PIAG }{
	Both fixed step-sizes $\frac{\alpha\gamma'}{\tau+1}$ and $\frac{\tau\gamma'}{(\tau+1)^2}$ in Proposition \ref{prop:stepsizeintegralbound} are comparable with those in the state-of-the-art works on PIAG and Async-BCD. Specifically, the largest fixed step-size in the literature is} \deleted{$\frac{h}{L(\tau+1/2)}$} \cite{Sun19,Deng20} \deleted{for PIAG} and for Async-BCD \deleted{, $\frac{h}{\hat{L}+2L\tau/\sqrt{m}}$} \cite{Davis16}, respectively. \replaced{This suggests that the adaptive step-size policies should be able to guarantee the same convergence rate. The next result shows that this is indeed the case.}{
	Proposition \ref{prop:stepsizeintegralbound} also guarantees the convergence of PIAG and Async-BCD with the two adaptive step-sizes.} 
	\begin{corollary}\label{Cor:adaptiveconverge}
		Suppose that $\tau_k\le \tau$ for all $k\in\mathbb{N}_0$ \replaced{and that the step-size is determined using either \eqref{eq:PIAGadapt} or \eqref{eq:adapt2}. Then}{. If either \eqref{eq:PIAGadapt} or \eqref{eq:adapt2} holds,}
		\begin{itemize}
			\item \replaced{for PIAG under the conditions of Theorem~\ref{theo:nonconvexPIAG}, in case (1)}{PIAG: under the corresponding conditions in Theorem \ref{theo:nonconvexPIAG}, \textbf{Non-convex}:} $\min_{1\le t\le k}\|\nabla f(x_t)+\xi_t\|^2=O(1/k)$, \replaced{in case (2) }{\textbf{Convex}:} $P(x_k)-P^\star=O(1/k)$, and \replaced{in case (3)}{\textbf{PL}:} $P(x_k)-P^\star\le O(\lambda^k)$ for some $\lambda\in (0,1)$.
			\item \replaced{for Async-BCD}{Async-BCD:} under the conditions in Theorem \ref{theo:coor}, $\min_{1\le t\le k}\|\tilde{\nabla} P(x_t)\|^2=O(1/k)$.
		\end{itemize}
	\end{corollary}
	\begin{proof} Immediate from Theorem 
		\deleted{The results are straightforward to see from Theorems} \ref{theo:nonconvexPIAG}--\ref{theo:coor} and Prop.~\ref{prop:stepsizeintegralbound}.
	\end{proof}
	Although the two adaptive step-sizes do not rely on the delay bound, the rates in Corollary \ref{Cor:adaptiveconverge} still reach the best-known order compared to related works on PIAG \cite{aytekin16,vanli2018global,Sun19,Deng20,Feyzmahdavian21} and Async-BCD \cite{Davis16,Sun17,liu2014,liu2015} that use such information in their step-sizes.

    \replaced{On the other hand, there are time-varying delays for which the adaptive step-sizes are guaranteed to perform much better than the fixed step-sizes. At the extreme, if the worst-case delay only occurs once and the system operates without delays afterwards (we call this a ``burst'' delay) the adaptive step-sizes will run with step-size $\gamma^{\prime}$. 
    The sum of step-sizes then tends to a value that is $\tau+1$ times larger than for the fixed step-sizes, with a corresponding speed-up.}

\deleted{
	\subsection{Superiority of the delay-adaptive step-sizes}\label{ssec:comparison}}
		\begin{figure*}[!htb]
		\caption{Comparison of delay-adaptive step-size and fixed step-size in delay models. The legends in (b),(c) follow those in (a).}
		\subfigure[constant delay]{\includegraphics[width=0.33\linewidth]{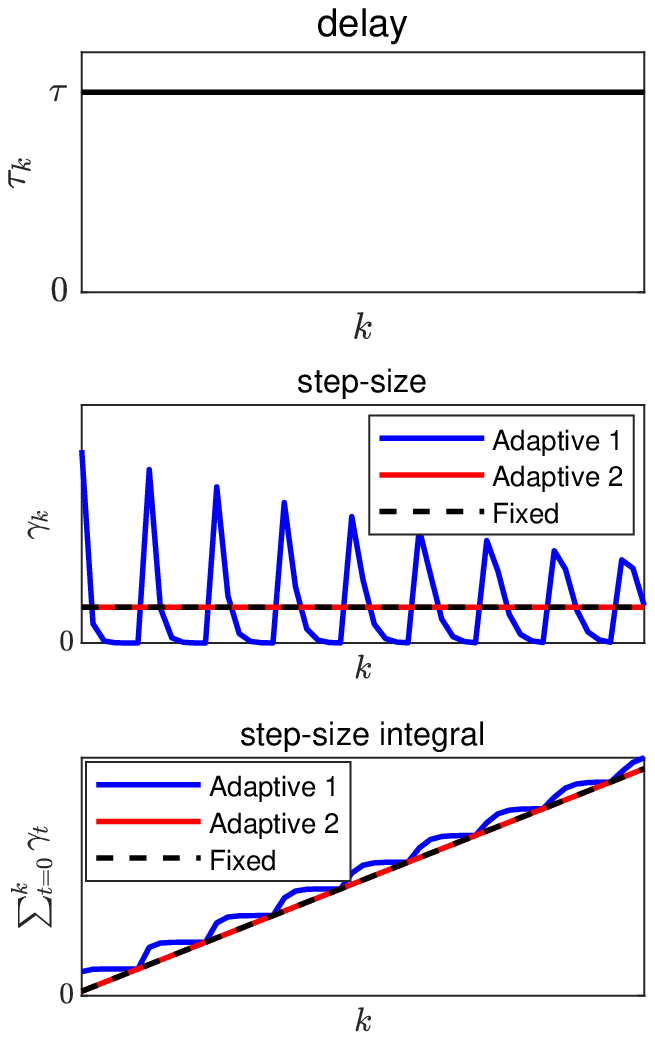}}
		\subfigure[random delay]{\includegraphics[width=0.33\linewidth]{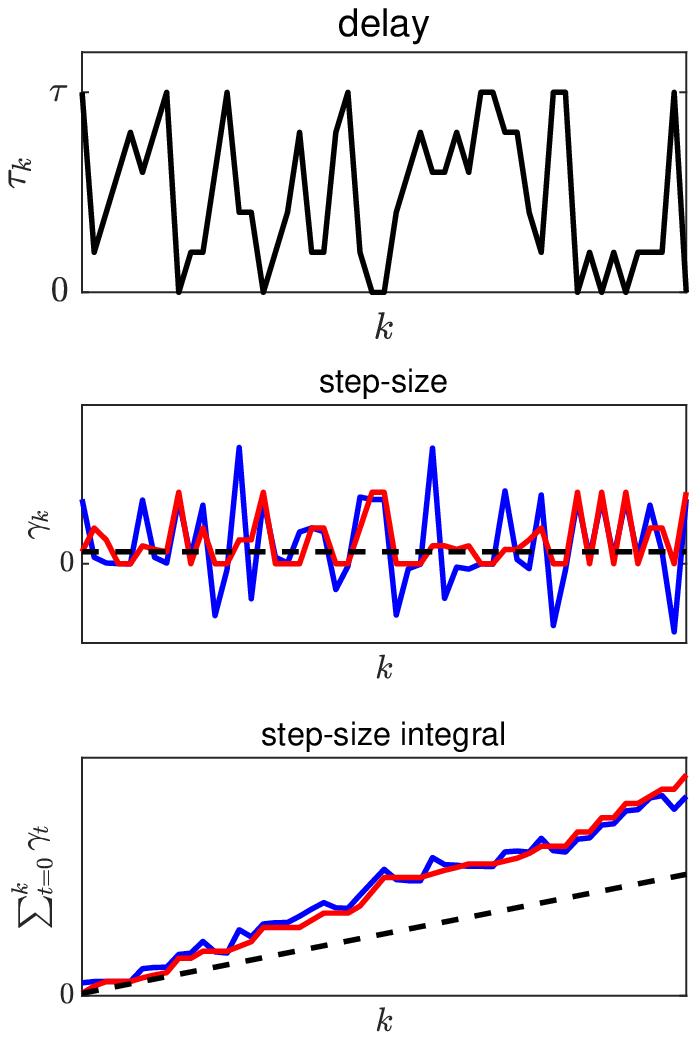}}
		\subfigure[burst delay]{\includegraphics[width=0.33\linewidth]{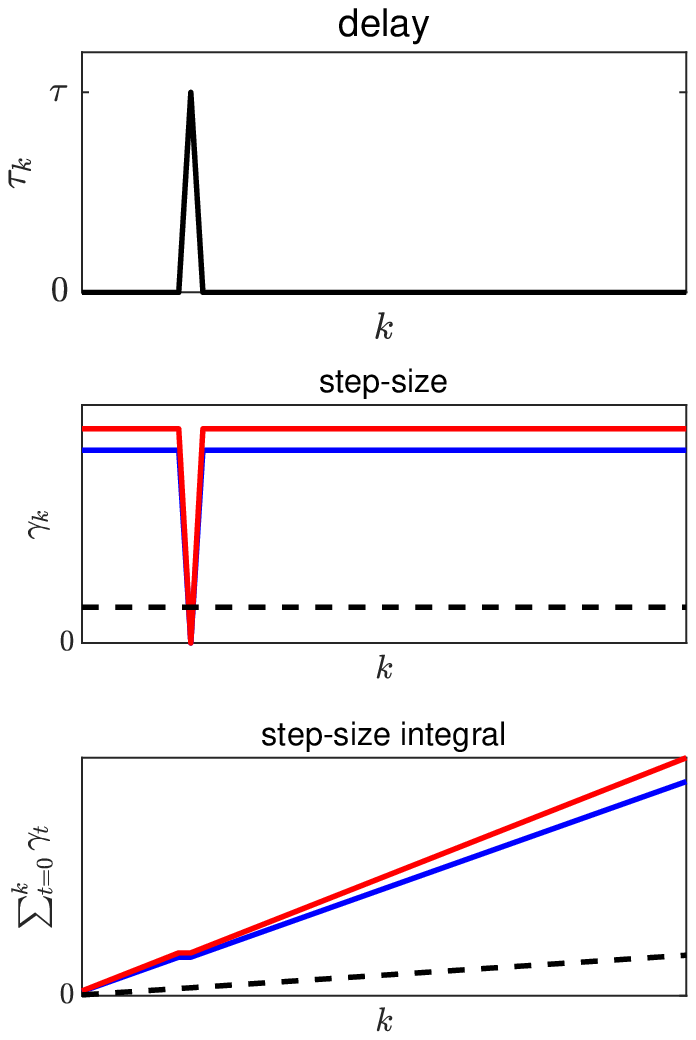}}
		\label{fig:delayallin1}
	\end{figure*}

\deleted{	
	Proposition \ref{prop:stepsizeintegralbound} derived for general bounded delay sequence only guarantees that the two adaptive step-sizes \eqref{eq:PIAGadapt}--\eqref{eq:adapt2} will not be worse than some fixed step-sizes. This is natural because for some particular delay sequences such as all the delays are equal to the maximal delay, it's difficult for delay-adaptive step-sizes to outperform fixed step-sizes, as the power of delay-adaptive step-sizes are their ability of adaption in different delays.} 
	
	\replaced{To obtain a more balanced comparison, we simulate the two adaptive step-size policies under the following delays:}{
	In this subsection, we illustrate the superiority of the two adaptive step-sizes by considering concrete delay models:}
	\begin{enumerate}[1)]
        \item constant: $\tau_k = \tau$.
		\item random: $\tau_k$ is drawn from $[0,\tau]$ uniformly at random.
		\item burst: $\tau_k=\tau$ at one epoch and $\tau_k=0$ otherwise.
	\end{enumerate}
	and compare these with the fixed step-size $\gamma_k=\gamma^{\prime}/(\tau+1)$. This step-size
	satisfies \eqref{eq:commonstepsizerule} and is comparable to state-of-the-art fixed step-sizes for PIAG and Async-BCD.
	
	We visualize the three delay models, the step-size $\gamma_k$, and the step-size integral $\sum_{t=0}^k \gamma_t$ in Figure \ref{fig:delayallin1}, in which we set $\alpha=0.9$ in Adaptive 1 and $\tau=5$ in all three models. We can make the following observations:

	\deleted{We have the following observations from Figure \ref{fig:delayallin1}:}
	\begin{itemize}
		\item In all three delay models, the sum of step-sizes for the two adaptive policies are at least similar to that of the fixed step-size, which validates Proposition \ref{prop:stepsizeintegralbound}.	
		\item The adaptive policies show the greatest superiority compared to the fixed step-size under the burst delay, where 
		the sum is asymptotically $\alpha(\tau+1)$ and $\tau+1$ times that of  the fixed step-size, respectively.
		\item When the proportion of small delays increases (constant$\rightarrow$random$\rightarrow$burst), so does the sum of step-sizes for the two delay-adaptive policies, reflecting their excellent adaption abilities to the true delay.
		\item Adaptive 2 is smoother and closer to its average behaviour than Adaptive 1, which often implies better robustness against noise.

	\end{itemize}

	\section{Numerical Experiments}
	
	Although the case for delay-adaptive step-sizes should be clear by now, we also demonstrate the end-effect on a simple machine learning problem. 
	We consider classification problem on the training data sets of rcv1 \cite{lewis2004rcv1} and MNIST \cite{deng2012mnist}, using the regularized logistic regression model: $f(x)= \frac{1}{N}\sum_{i=1}^N\left(\log(1+e^{-b_i(a_i^Tx)})+\frac{\lambda_2}{2}\|x\|^2\right)$, $R(x) = \lambda_1\|x\|_1$, where $a_i$ is the feature of the $i$th sample, $b_i$ is the corresponding label, and $N$ is the number of samples. We pick $(\lambda_1,\lambda_2)=(10^{-5},10^{-4})$ for rcv1 and $(\lambda_1,\lambda_2)=(10^{-3},10^{-4})$ for MNIST.
	We \replaced{run}{employ} both PIAG and Async-BCD \deleted{to solve the regression problem} and compare the performance of delay-adaptive step-sizes and fixed step-sizes. 
	In the first adaptive policy (Adaptive 1), we let $\alpha=0.9$.

	\subsection{PIAG}
	
	We split the samples in each data set into $n=10$ batches and assign each batch to a single worker. We run the method on a machine with Intel Xeon Silver 4210 Processor including 10-cores and 20 threads, where 1 thread is server and 10 threads are workers.

	We compare the two delay-adaptive step-sizes with $\gamma'=\frac{h}{L}$ with the fixed step-size $\gamma_k=\frac{h}{L(\tau+1/2)}$ in \cite{Sun19,Deng20}, where $h=0.99$ in all three step-sizes. At each iteration, only one batched gradient is updated, i.e., $\mathcal{R}=1$ in Algorithm \ref{alg:PIAG}. The distributions of the generated $\{\tau_k\}$ are plotted in Figure \ref{fig:delay}(a), where the maximal delays for rcv1 and MNIST are $75$ and $74$, respectively, and are much larger than most $\tau_k$'s (over $92\%$ $\tau_k$'s are smaller than or equal to $25$ for both data sets). Moreover, the workers have different maximal delays varying in $[31,75]$ and $[34,74]$ for rcv1 and MNIST, respectively, reflecting differences in their computing power.

	\begin{figure}[!htb]
		\centering
		\caption{Convergence of PIAG}
		\vspace{0.2cm}
		
		\subfigure[rcv1]{
			\includegraphics[scale=0.8]{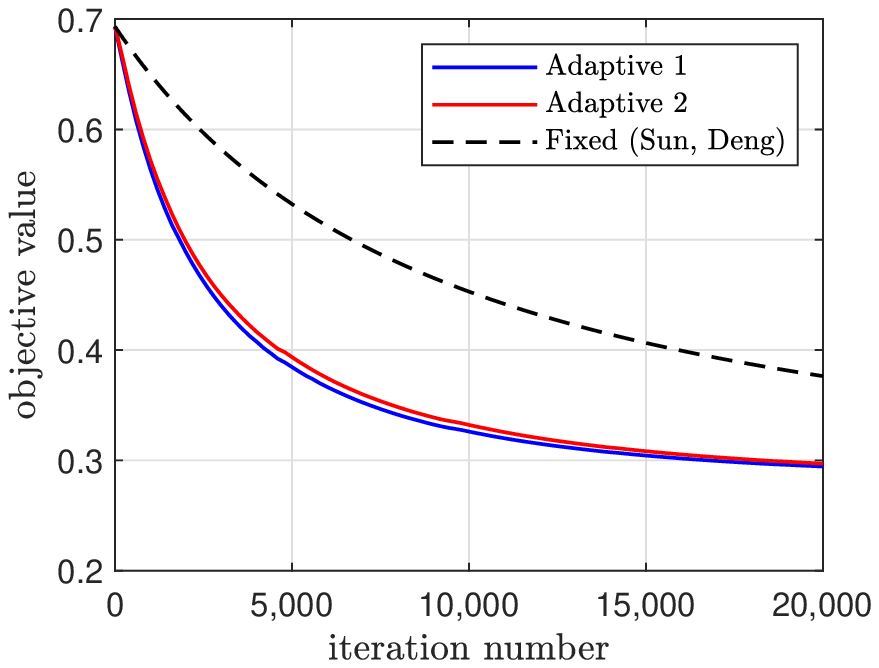}\label{fig:PIAGrcv1}}
		\subfigure[MNIST]{
			\includegraphics[scale=0.8]{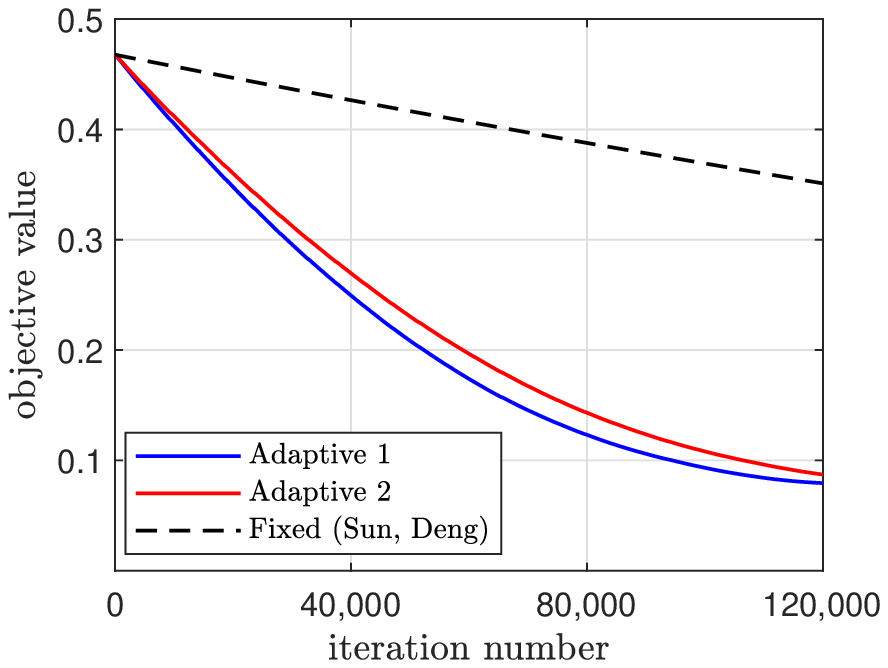}\label{fig:PIAGmnist}}
		\label{fig:PIAG}
	\end{figure}
	
	The objective error of PIAG with the three step-sizes is shown in Figure \ref{fig:PIAG}. Clearly, PIAG converges much faster under the delay-adaptive step-sizes than under the fixed step-size on both data sets. For example in Figure \ref{fig:PIAG}(a), compared to the fixed step-size, PIAG with Adaptive 1 and Adaptive 2 only need approximately $1/3$ and $1/2$ the number of iterations, respectively, to achieve objective value of $0.3$. This demonstrates the effectiveness of the our adaptive policies.
	
	\begin{figure}[!ht]
		\centering
		\caption{Delay distribution}
		\vspace{0.2cm}
        \subfigure[PIAG (10 workers)]{
			\includegraphics[scale=0.8]{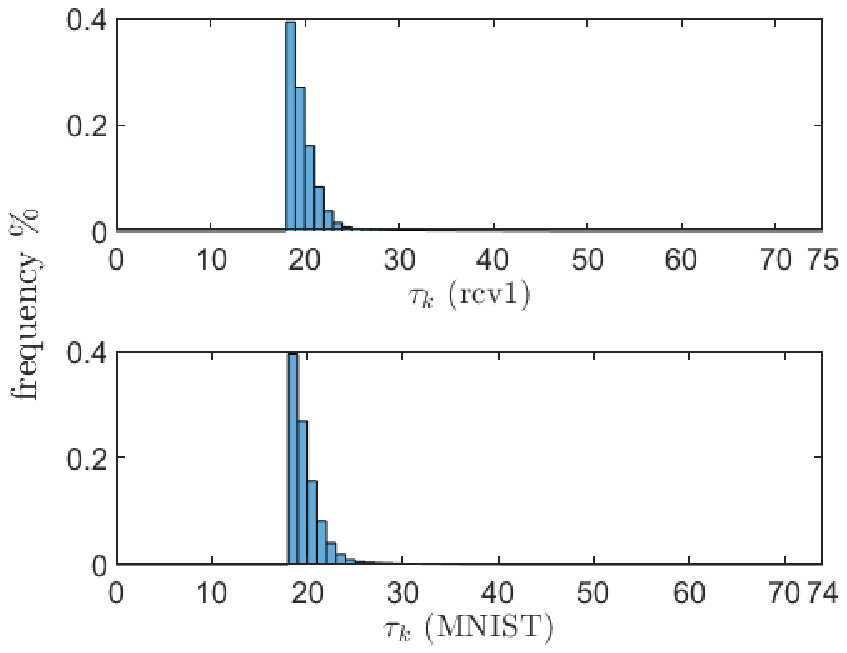}\label{fig:delayPIAG}}~~
		\subfigure[Async-BCD (8 workers)]{
    	\includegraphics[scale=0.8]{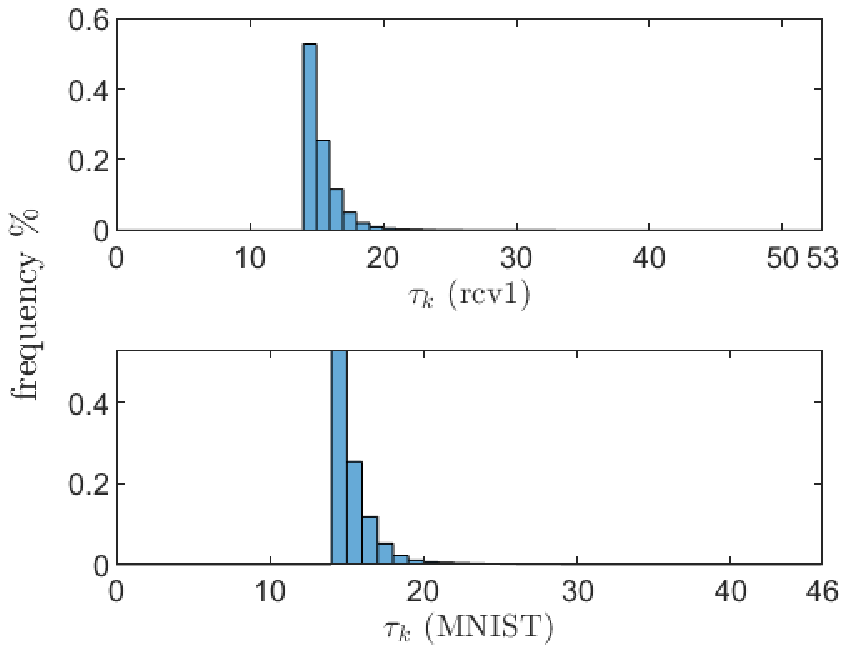}\label{fig:delayBCD}}
		\label{fig:delay}
	\end{figure}
	
	\subsection{Async-BCD}
	
	We use $n=8$ workers and split the variable $x$ into $m=20$ blocks almost evenly, with some blocks having one dimension more than the other blocks. We implement the algorithm on threads 1-8 of a machine with Intel Xeon Silver 4210 Processor including 10-cores and 20 threads.
	
	We compare the two delay-adaptive step-sizes with $\gamma'=\frac{h}{\hat{L}}$ with fixed step-sizes $\gamma_k=\frac{h}{L(\tau+1/2)}$ in \cite{Sun17} and $\gamma_k=\frac{h}{\hat{L}+2L\tau/\sqrt{m}}$ in \cite{Davis16}. In all cases, $h=0.99$. 
	
	\begin{figure}[!ht]
		\centering
		\caption{Convergence of Async-BCD.}
		\vspace{0.2cm}
        \subfigure[rcv1]{
			\includegraphics[scale=0.8]{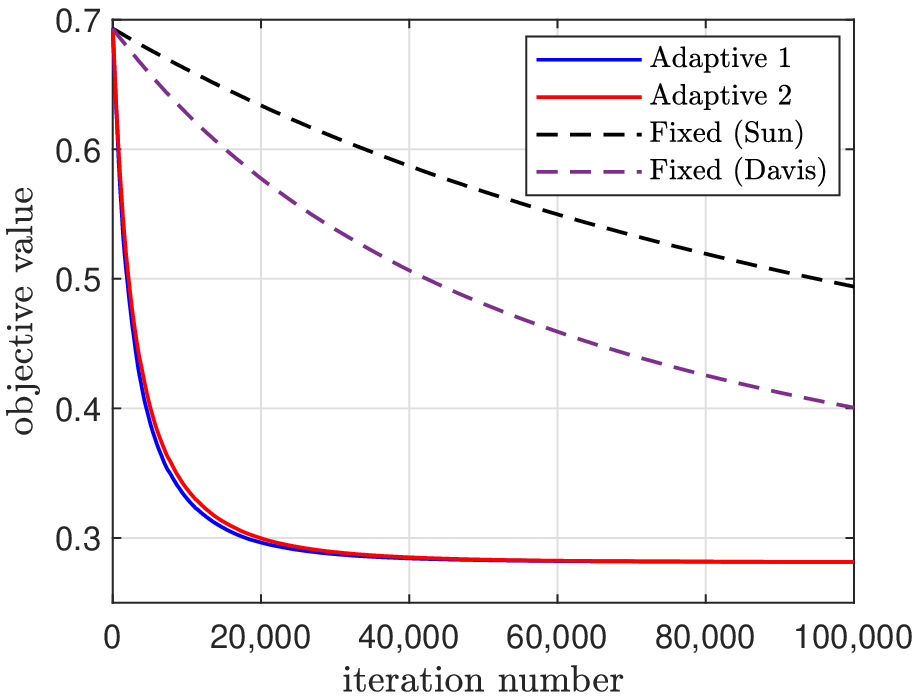}\label{fig:BCDrcv}}~~
		\subfigure[MNIST]{
    	\includegraphics[scale=0.8]{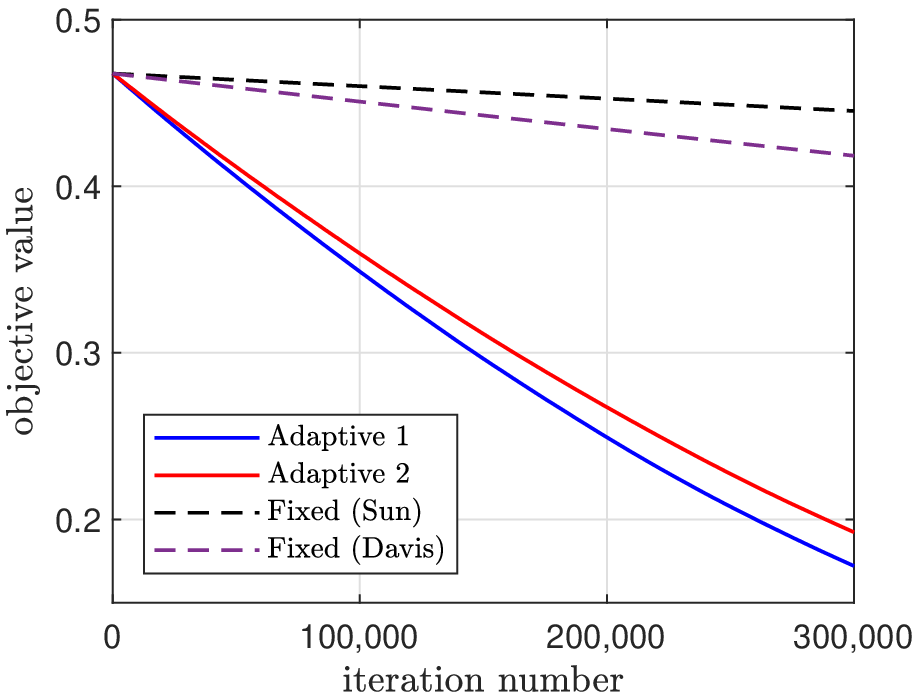}\label{fig:BCDmnist}}
		\label{fig:asyncbcd}
	\end{figure}
	
	Figure \ref{fig:asyncbcd} plots the objective error of Async-BCD with the aforementioned step-sizes. For both datasets, Async-BCD needs substantially longer time to converge under the fixed step-sizes than under the adaptive policies. \deleted{to achieve the same accuracy, Async-BCD with the two fixed step-sizes need a much larger number of iterations compared to Async-BCD with delay-adaptive step-size.} This exhibits once again the advantages of our delay-adaptive step-sizes. The distributions of the generated $\{\tau_k\}$ for the two data sets are plotted in Figure \ref{fig:delay}(b), where the maximal delays are $53$ and $46$ for rcv1 and MNIST, respectively, and are much larger than most $\tau_k$'s (over $97\%$ $\tau_k$'s are smaller than or equal to $20$ for both data sets). Moreover, the workers have different maximal delays varying in $[32,53]$ and $[29,46]$ for rcv1 and MNIST, respectively, due to their different computational capabilities.

	\section{Conclusions}
    We have shown that it is possible to design, implement and analyze asynchronous optimization algorithms that adapt to the true system delays. This is a significant departure from the state-of-the art, that rely on an (often conservative) upper bound of the system delays and use fixed learning rates that are tuned to the worst-case situation. 
    
    Although many of the principles that we have put forward apply to broad classes of algorithms and systems, we have provided detailed treatments of two specific algorithms: PIAG and Async-BCD. Explicit convergence rate bounds and numerical experiments on different data sets and delay traces demonstrate substantial advantages over the state-of-the art. 

    Future work includes developing delay-adaptive step-sizes for other asynchronous algorithms such as Asynchronous SGD \cite{dean2012large,zhang2020taming} and extending the adaptive mechanism to also estimate the Lipschitz constant (and possibly other parameters) on-line.

	\section*{Appendices}
	\appendix
	\section{Proof of Theorem \ref{theo:sequence}}\label{proof:theoseq}
	
	Dividing both sides of \eqref{eq:asynseq} by $Q_{k+1}$ and summing the resulting equation from $k=0$ to $k=K-1$, we obtain
	\begin{equation}\label{eq:sumF0toKminus1}
		\begin{split}
			\frac{V_K}{Q_K}+\sum_{k=1}^K \frac{X_k}{Q_k} \le& V_0 + \sum_{k=0}^{K-1}\left(-\frac{r_k}{Q_{k+1}}W_k+\frac{p_k}{Q_{k+1}}\sum_{\ell=k-\tau_k}^{k-1} W_\ell\right).
		\end{split}
	\end{equation}
	Define $\mathcal{T}_\ell:=\{t\in\mathbb{N}_0: \ell\in [t-\tau_t, t-1]\}$. We have
	\begin{equation}\label{eq:summationtransform}
		\begin{split}
			\sum_{k=0}^{K-1}\frac{p_k}{Q_{k+1}}\sum_{\ell=k-\tau_k}^{k-1} W_\ell &\le\sum_{\ell=0}^{K-2}\left(\sum_{k\in \mathcal{T}_\ell}\frac{p_k}{Q_{k+1}}\right)W_\ell.
		\end{split}
	\end{equation}
	To see \eqref{eq:summationtransform}, note that in the left-hand side, $W_\ell$ occurs only if $\ell\in \{0,1,\ldots,K-2\}$ and $\frac{p_k}{Q_{k+1}}W_\ell$ occurs only if $k\in\mathcal{T}_\ell$. Fix $\ell\in\mathbb{N}_0$. For any $k\in\mathcal{T}_\ell$, because $\ell\in [k-\tau_k,k-1]$, either $p_k=0$ or
	\begin{equation*}
		\frac{p_k}{Q_{k+1}}\le \frac{r_\ell}{Q_{\ell+1}}-\sum_{t=\ell+1}^{k-1} \frac{p_t}{Q_{t+1}}.
	\end{equation*}
	Let $k':=\max\{k\in \mathcal{T}_\ell: p_k>0\}$. By the above equation,
	\begin{equation*}
		\frac{r_\ell}{Q_{\ell+1}}\ge \sum_{t=\ell+1}^{k'} \frac{p_t}{Q_{t+1}} \ge \sum_{t\in \mathcal{T}_\ell} \frac{p_t}{Q_{t+1}} .
	\end{equation*}
	Substituting \eqref{eq:summationtransform} and the above equation into \eqref{eq:sumF0toKminus1} gives
	\begin{equation*}
		\frac{V_K}{Q_K}+\sum_{k=1}^K \frac{X_k}{Q_k}\le V_0-\frac{r_{K-1}}{Q_K}W_{K-1}\le V_0,
	\end{equation*}
	which derives the result.
	
	\section{Proof of Theorem \ref{theo:nonconvexPIAG}}\label{proof:theoPIAG}
	For any $k\in\mathbb{N}$, if $\gamma_{k-1}>0$, then $$\xi_k=-\frac{1}{\gamma_{k-1}}(x_k-x_{k-1})-g_{k-1}.$$ Otherwise, $\xi_k$ can be any subgradient of $R$ at $x_k$. By the first-order optimality condition of \eqref{eq:update},
	\begin{equation}\label{eq:sDefinition}
		\xi_k\in \partial R(x_k),~\forall k\in \mathbb{N}.
	\end{equation}
	The proof mainly uses Theorem \ref{theo:sequence} and the following lemma, which shows that some quantities in PIAG satisfy the asynchronous sequence \eqref{eq:asynseq}.
	\begin{lemma}\label{lemma:PIAGsequence}
		Suppose that all the conditions in Theorem \ref{theo:nonconvexPIAG} hold. Then, the asynchronous sequence \eqref{eq:asynseq} holds with
		\begin{equation*}
			W_k = 
			\begin{cases}
				\frac{1}{\gamma_k}\|x_{k+1}-x_k\|^2, & \gamma_k>0,\\
				0, & \gamma_k=0,
			\end{cases}~~\forall k\in\mathbb{N}_0,
		\end{equation*}
		and
		\begin{enumerate}[(1)]
			\item \textbf{Non-convex}:
			\begin{align*}
				& X_{k+1} = \frac{\gamma_k}{2}\frac{1-h}{h^2-h+1}\|\nabla f(x_{k+1})+\xi_{k+1}\|^2,\\
				& V_k = P(x_k)-P^\star,~p_k=\frac{\gamma_k hL}{2},\\
				& q_k=1,~r_k = \frac{h^2}{2}-p_k.
			\end{align*}
			\item \textbf{Convex}: If each $f^{(i)}$ is convex, then
			\begin{align*}
				& X_k = 0,~V_k = a_k(P(x_k)-P^\star)+\frac{1}{2}\|x_k-x^\star\|^2,\\
				& p_k = \frac{\gamma_k}{2}(a_kL+1),~r_k=\frac{a_k}{2}-p_k,~q_k=1,
			\end{align*}
			where $a_k=\frac{h(h+1)}{L(1-h)}+\sum_{\ell=0}^{k-1} \gamma_\ell$.
			\item\textbf{PL}: If $P$ satisfies the proximal PL-condition \eqref{eq:proximalPLcond}, then
			\begin{align*}
				& V_k = P(x_k)-P^\star,~q_k=\frac{1}{1+\frac{c\sigma(1-\tilde{h})}{\tilde{h}^2-\tilde{h}+1}\gamma_k},\\
				& X_{k+1} = 0,~p_k=\frac{\gamma_k\tilde{h}L}{2},~r_k = \frac{q_k\tilde{h}^2}{2}-\frac{\gamma_k\tilde{h}L}{2},
			\end{align*}
			where $\tilde{h} = \frac{1+h}{2}$ and $c = \min\left(1,\frac{1-h}{2h}\frac{L}{\sigma}\right)$.
		\end{enumerate}
		In all the three cases, either $p_k=0$ or \eqref{eq:generalprincipal} holds.
	\end{lemma}
	
	Using Lemma \ref{lemma:PIAGsequence} and Theorem \ref{theo:sequence}, the result on nonconvex and convex case in Theorem \ref{theo:nonconvexPIAG} is straightforward. To see the proximal-PL case in Theorem \ref{theo:nonconvexPIAG}, note that because $c\le 1$, $\gamma_k\le \gamma'=\frac{h}{L}\le \frac{\tilde{h}}{L}$, and $\sigma\le L$,
	\begin{equation*}
	    \frac{c\sigma(1-\tilde{h})}{\tilde{h}^2-\tilde{h}+1}\gamma_k\le\frac{c\sigma}{L}\frac{\tilde{h}(1-\tilde{h})}{\tilde{h}^2-\tilde{h}+1}\le \frac{\tilde{h}(1-\tilde{h})}{\tilde{h}^2-\tilde{h}+1}\le \frac{1}{3}.
	\end{equation*}
	In addition, for any $\epsilon\in (0,1/3]$, $\frac{1}{1+\epsilon}\le 1-\frac{3}{4}\epsilon\le e^{-\frac{3}{4}\epsilon}$. Therefore,
	\begin{equation*}
	    \frac{1}{1+\frac{c\sigma(1-\tilde{h})}{\tilde{h}^2-\tilde{h}+1}\gamma_k}\le e^{-\frac{3}{4}\frac{c\sigma(1-\tilde{h})}{\tilde{h}^2-\tilde{h}+1}\gamma_k},
	\end{equation*}
	which further gives
	\begin{equation*}
	    Q_k\le e^{-\frac{3c\sigma(1-\tilde{h})}{4(\tilde{h}^2-\tilde{h}+1)}\sum_{t=0}^{k-1}\gamma_t}.
	\end{equation*}
	Using the above equation and \eqref{eq:Vconverge}, we obtain the result.
	
	\subsection{Proof of Lemma \ref{lemma:PIAGsequence}}
	When $\gamma_k = 0$, $p_k=0$ in all three cases. Below, we assume $\gamma_k>0$ and prove \eqref{eq:generalprincipal} in all the three cases.
	
	\subsubsection{Proof of the nonconvex case}
	We first prove that for any $k\in\mathbb{N}_0$,
	\begin{equation}\label{eq:alphabetadef}
		\begin{split}
			P(x_{k+1})-P^\star-(P(x_k)-P^\star) \le&\frac{1}{2}\gamma_khL\sum_{j=k-\tau_k}^{k-1}W_j-\frac{\gamma_k}{2}\|\nabla f(x_k)+\xi_{k+1}\|^2 -\frac{1-\gamma_kL}{2}W_k.
		\end{split}
	\end{equation}

	By \eqref{eq:sDefinition} and the convexity of $R$,
	\begin{equation}\label{eq:Rconvex}
		R(x_{k+1})-R(x_k) \le \langle \xi_{k+1}, x_{k+1}-x_k\rangle.
	\end{equation}
	Moreover, $f$ is $L-$smooth due to the $L_i$-smoothness of each $f^{(i)}$. Then,
	\begin{equation}\label{eq:Fsmooth}
		\begin{split}
			f(x_{k+1})-f(x_k)\le&\langle \nabla f(x_k), x_{k+1}-x_k\rangle+\frac{L}{2}\|x_{k+1}-x_k\|^2.
		\end{split}
	\end{equation}
	By adding \eqref{eq:Rconvex} and \eqref{eq:Fsmooth}, we have
	\begin{equation}\label{eq:Pdescent}
		\begin{split}
			P(x_{k+1})-P(x_k) \le\langle \nabla f(x_k)+\xi_{k+1}, x_{k+1}-x_k\rangle+\frac{L}{2}\|x_{k+1}-x_k\|^2,
		\end{split}
	\end{equation}
	where
	\begin{equation}\label{eq:innerproductexpansion}
		\begin{split}
			&\langle \nabla f(x_k)+\xi_{k+1}, x_{k+1}-x_k\rangle= \gamma_k\langle \nabla f(x_k)+\xi_{k+1}, \frac{1}{\gamma_k}(x_{k+1}-x_k)\rangle\\
			=& \frac{\gamma_k}{2}\|\nabla f(x_k)+\xi_{k+1}+\frac{1}{\gamma_k}(x_{k+1}-x_k)\|^2-\frac{\gamma_k}{2}\|\nabla f(x_k)+\xi_{k+1}\|^2-\frac{1}{2\gamma_k}\|x_{k+1}-x_k\|^2.
		\end{split}
	\end{equation}
	From the definition of $\xi_{k+1}$, we have
	\begin{equation*}
		\nabla f(x_k)+\xi_{k+1}+\frac{1}{\gamma_k}(x_{k+1}-x_k) = \nabla f(x_k)-g_k.
	\end{equation*}
	Substituting this equation into \eqref{eq:innerproductexpansion} yields
	\begin{equation}\label{eq:pk1pkboundedbygraFgdiff}
		\begin{split}
			\langle \nabla f(x_k)+\xi_{k+1}, x_{k+1}-x_k\rangle\le& \frac{\gamma_k}{2}\|\nabla f(x_k)-g_k\|^2-\frac{\gamma_k}{2}\|\nabla f(x_k)+\xi_{k+1}\|^2-\frac{1}{2\gamma_k}\|x_{k+1}-x_k\|^2.
		\end{split}
	\end{equation}
	By the $L_i$-smoothness of each $f^{(i)}$ and the definition of $g_k$,
	\begin{equation}\label{eq:adaptdiffupperbound}
		\begin{split}
			\|\nabla f(x_k)- g_k\|^2&= \|\frac{1}{n}\sum_{i=1}^n (\nabla f^{(i)}(x_k) - \nabla f^{(i)}(x_{k-\tau_k^{(i)}}))\|^2\\
			&= \frac{1}{n^2}\|\sum_{i=1}^n (\nabla f^{(i)}(x_k) - \nabla f^{(i)}(x_{k-\tau_k^{(i)}}))\|^2\\
			&\le \frac{1}{n}\sum_{i=1}^n \|\nabla f^{(i)}(x_k) - \nabla f^{(i)}(x_{k-\tau_k^{(i)}})\|^2\\
			&\le \frac{1}{n}\sum_{i=1}^n L_i^2\|x_k-x_{k-\tau_k^{(i)}}\|^2\\
			&= \frac{1}{n}\sum_{i=1}^n L_i^2\|\sum_{j=k-\tau_k^{(i)}}^{k-1}(x_{j+1}-x_j)\|^2.
		\end{split}
	\end{equation}
	In addition,
	\begin{equation}\label{eq:adaptgammagjsj}
		\begin{split}
			\|\sum_{j=k-\tau_k^{(i)}}^{k-1}(x_{j+1}-x_j)\|^2\le& (\sum_{j=k-\tau_k^{(i)}}^{k-1}\|x_{j+1}-x_j\|)^2= (\sum_{j=k-\tau_k^{(i)}}^{k-1} \sqrt{\gamma_jW_j})^2\allowdisplaybreaks\\
			\le& (\sum_{j=k-\tau_k^{(i)}}^{k-1}\sqrt{\gamma_j}^2)(\sum_{j=k-\tau_k^{(i)}}^{k-1} \sqrt{W_j}^2)\allowdisplaybreaks\\
			= &(\sum_{j=k-\tau_k^{(i)}}^{k-1} \gamma_j)(\sum_{j=k-\tau_k^{(i)}}^{k-1}W_j) \le\frac{h}{L}\sum_{j=k-\tau_k}^{k-1}W_j,
		\end{split}
	\end{equation}
	where the second inequality is due to the Cauchy–Schwarz inequality and the last step is due to $\tau_k^{(i)}\le \tau_k$, \eqref{eq:commonstepsizerule} with $\gamma'=\frac{h}{L}$, and $\gamma_k>0$. By \eqref{eq:adaptdiffupperbound}, \eqref{eq:adaptgammagjsj}, and $\frac{1}{n}\sum_{i=1}^n L_i^2 = L^2$,
	\begin{equation*}
		\begin{split}
			&\|\nabla f(x_k)- g_k\|^2 \le hL\sum_{j=k-\tau_k}^{k-1}W_j.
		\end{split}
	\end{equation*}
	Combining the above equation with \eqref{eq:pk1pkboundedbygraFgdiff} and \eqref{eq:Pdescent} yields \eqref{eq:alphabetadef}.

	Next, we use \eqref{eq:alphabetadef} to derive \eqref{eq:asynseq}. The equation \eqref{eq:alphabetadef} can be rewritten as
	\begin{equation}\label{eq:newalphabetadef}
		\begin{split}
			&P(x_{k+1})-P^\star-(P(x_k)-P^\star)\\
			\le&\frac{1}{2}\gamma_khL\sum_{j=k-\tau_k}^{k-1}W_j-\frac{\gamma_k}{2}\|\nabla f(x_k)+\xi_{k+1}\|^2-\frac{1-h^2-(1-h)\gamma_kL}{2}W_k-\frac{h^2-\gamma_khL}{2}W_k.
		\end{split}
	\end{equation}

	Because of \eqref{eq:commonstepsizerule} with $\gamma'=\frac{h}{L}$, we have $\gamma_k\le \frac{h}{L}$ and
	\begin{equation}\label{eq:nablaFpluss}
		\begin{split}
			&\frac{\gamma_k}{2}\|\nabla f(x_k)\!+\!\xi_{k+1}\|^2\!+\!\frac{1-h^2\!-\!(1-h)\gamma_kL}{2}W_k\allowdisplaybreaks\\
			=& \frac{\gamma_k}{2}(\|\nabla f(x_k)+\xi_{k+1}\|^2+\frac{1-h^2-(1-h)\gamma_kL}{\gamma_k^2L^2}L^2\|x_{k+1}-x_k\|^2)\allowdisplaybreaks\\
			\ge&\frac{\gamma_k}{2}(\|\nabla f(x_k)\!+\!\xi_{k+1}\|^2\!+\!\frac{1\!-\!h}{h^2}L^2\|x_{k+1}\!-\!x_k\|^2),
		\end{split}
	\end{equation}
	where the last step is due to $1-h^2-(1-h)\gamma_k L\ge 1-h^2-(1-h)h=1-h>0$ and $\gamma_k^2L^2\le h^2$.
	By the $L$-smoothness of $f$ and the AM-GM inequality, for any $\eta>0$,
	\begin{equation*}
		\begin{split}
			&\|\nabla f(x_{k+1})+\xi_{k+1}\|^2\\
			=&\|(\nabla f(x_{k+1})-\nabla f(x_k))+(\nabla f(x_k)+\xi_{k+1})\|^2\\
			\le& (1+\eta)\|\nabla f(x_k)+\xi_{k+1}\|^2+(1+1/\eta)\|\nabla f(x_{k+1})-\nabla f(x_k)\|^2\\
			\le& (1+\eta)\|\nabla f(x_k)+\xi_{k+1}\|^2+(1+1/\eta)L^2\|x_{k+1}-x_k\|^2.
		\end{split}
	\end{equation*}
	Letting $\eta = \frac{h^2}{1-h}$ in the above equation, we have
	\begin{equation*}
		\begin{split}
			&\frac{1-h}{h^2-h+1}\|\nabla f(x_{k+1})+\xi_{k+1}\|^2\le\|\nabla f(x_k)+\xi_{k+1}\|^2+\frac{(1-h)L^2}{h^2}\|x_{k+1}-x_k\|^2,
		\end{split}
	\end{equation*}
	which, together with \eqref{eq:nablaFpluss} and \eqref{eq:newalphabetadef}, gives
	\begin{equation}\label{eq:nonconvexcasePIAGseq}
	\begin{split}
	    	&P(x_{k+1})-P^\star-(P(x_k)-P^\star)\\
			\le&\frac{1}{2}\gamma_khL\sum_{j=k-\tau_k}^{k-1}W_j-\frac{h^2-\gamma_khL}{2}W_k-\frac{\gamma_k}{2}\frac{1-h}{h^2-h+1}\|\nabla f(x_{k+1})+\xi_{k+1}\|^2,
	\end{split}    
	\end{equation}
	i.e., \eqref{eq:asynseq} holds.
	
	Finally, it is straightforward to see that \eqref{eq:commonstepsizerule} with $\gamma'=h/L$ guarantees \eqref{eq:generalprincipal}.
	
	\subsubsection{Proof of the convex case}
	Define $a_0=\frac{h(h+1)}{L(1-h)}$ and $a_k=a_0+\sum_{\ell=0}^{k-1} \gamma_\ell$ for all $k\in\mathbb{N}$. Multiplying both sides of \eqref{eq:alphabetadef} by $a_k$ and using $h<1$ gives
	\begin{equation}\label{eq:convexalphabetadef}
		\begin{split}
			a_k(P(x_{k+1})-P^\star)-a_k(P(x_k)-P^\star)\le &\frac{a_k\gamma_kL}{2}\sum_{j=k-\tau_k}^{k-1}W_j-\frac{a_k(1-\gamma_kL)}{2}W_k.
		\end{split}
	\end{equation}
	In addition, using a similar derivation of equation (47) in \cite{Feyzmahdavian21}, we have
	\begin{equation*}
		\begin{split}
			&\gamma_k(P(x_{k+1})-P^\star)\le \gamma_k\sum_{i=1}^n\frac{L_i}{2n}\|x_k-x_{k-\tau_k^{(i)}}\|^2+\frac{1}{2}(\|x_k-x^\star\|^2-\|x_{k+1}-x^\star\|^2-\|x_{k+1}-x_k\|^2),
		\end{split}
	\end{equation*}
	which, together with \eqref{eq:adaptgammagjsj}, $h< 1$, and $\frac{1}{n}\sum_{i=1}^n L_i \le \sqrt{\frac{1}{n}\sum_{i=1}^n L_i^2} = L$, leads to
	\begin{equation}\label{eq:gammaFxkminusFstar}
		\begin{split}
			&\gamma_k(P(x_{k+1})-P^\star)\le \frac{1}{2}\gamma_k\sum_{j=k-\tau_k}^{k-1}W_j-\frac{1}{2}\gamma_k W_k+\frac{1}{2}(\|x_k-x^\star\|^2-\|x_{k+1}-x^\star\|^2).
		\end{split}
	\end{equation}
	Adding \eqref{eq:gammaFxkminusFstar} with \eqref{eq:convexalphabetadef} gives
	\begin{equation*}
		\begin{split}
			&a_{k+1}(P(x_{k+1})-P^\star)+\frac{1}{2}\|x_{k+1}-x^\star\|^2-(a_k(P(x_k)-P^\star)+\frac{1}{2}\|x_k-x^\star\|^2)\\
			\le& \frac{\gamma_k}{2}(a_kL+1)\sum_{j=k-\tau_k}^{k-1}W_j-\frac{1}{2}(a_k(1-\gamma_k L)+\gamma_k)W_k\\
			=& \frac{\gamma_k}{2}(a_kL+1)\sum_{j=k-\tau_k}^{k-1}W_j-\frac{1}{2}(a_k-\gamma_k(a_kL+1))W_k,
		\end{split}
	\end{equation*}
	i.e., \eqref{eq:asynseq} holds.
	
	Finally, we prove \eqref{eq:generalprincipal}. In this case, \eqref{eq:generalprincipal} reduces to
	\begin{equation*}
		\sum_{j=\ell}^k \gamma_j(a_jL+1)\le a_\ell, \forall \ell\in [k-\tau_k, k],\quad\forall k\in\mathbb{N}_0.
	\end{equation*}
	Since $a_k$ is monotonically non-decreasing, the above equation is equivalent to
	\begin{equation}\label{eq:PIAGconvexinproofprincipal}
		\sum_{j=k-\tau_k}^k \gamma_j(a_jL+1)\le a_{k-\tau_k},\quad\forall k\in\mathbb{N}_0.
	\end{equation}
	Since for each $j\in [k-\tau_k,k]$, $a_j\le a_k = a_{k-\tau_k}+\sum_{j=k-\tau_k}^{k-1} \gamma_t\le a_{k-\tau_k}+\frac{h}{L}$ by \eqref{eq:commonstepsizerule} and $\gamma'=\frac{h}{L}$, we have
	\begin{equation*}
		\begin{split}
			\sum_{j=k-\tau_k}^k \gamma_j(a_jL+1) &\le \frac{h(a_{k-\tau_k}L+h+1)}{L}.
		\end{split}
	\end{equation*}
	Moreover, since $a_{k-\tau_k}\ge a_0$ and $\frac{La_{k-\tau_k}}{h(a_{k-\tau_k}L+h+1)}$ increases at $a_{k-\tau_k}$,
	\begin{equation*}
		\frac{La_{k-\tau_k}}{h(a_{k-\tau_k}L+h+1)}\ge \frac{La_0}{h(a_0L+h+1)} = 1.
	\end{equation*}
	Combining the above two equations, we have \eqref{eq:PIAGconvexinproofprincipal}, so that \eqref{eq:generalprincipal} also holds.
	
	\subsubsection{Proof of the proximal PL case}
	By Appendix G in \cite{karimi16}, \eqref{eq:proximalPLcond} is equivalent to:
	\begin{equation}\label{eq:equivPL}
		\sigma(P(x)-P^\star)\le \frac{\|s\|^2}{2},~\forall s\in\partial P(x),~\forall x\in\R^d.
	\end{equation}
	Because $\sum_{t=k-\tau_k}^k \gamma_t\le \frac{h}{L}\le \frac{\tilde{h}}{L}$, \eqref{eq:nonconvexcasePIAGseq} with $h$ being replaced by $\tilde{h}$ also holds by its derivation, which, together with \eqref{eq:equivPL} and $c\le 1$, yields
	\begin{equation}\label{eq:PIAGPLproof}
		\begin{split}
			&(1\!+\!\frac{c\sigma(1\!-\!\tilde{h})\gamma_k}{\tilde{h}^2\!-\!\tilde{h}\!+\!1})(P(x_{k+1})\!-\!P^\star)\!-\!(P(x_k)-P^\star)\\
			&\le\frac{1}{2}\gamma_k\tilde{h}L\sum_{j=k-\tau_k}^{k-1}W_j-\frac{\tilde{h}^2-\gamma_k\tilde{h}L}{2}W_k.
		\end{split}
	\end{equation}
	
	Dividing both sides of \eqref{eq:PIAGPLproof} by $1+\frac{c\sigma(1-\tilde{h})\gamma_k}{\tilde{h}^2-\tilde{h}+1}$ ensures
	\begin{equation*}
		\begin{split}
			&P(x_{k+1})-P^\star - \frac{P(x_k)-P^\star}{1+\frac{c\sigma(1-\tilde{h})\gamma_k}{\tilde{h}^2-\tilde{h}+1}}\\
			\le & \frac{\frac{1}{2}\gamma_k\tilde{h}L\sum_{j=k-\tau_k}^{k-1}W_j-\frac{\tilde{h}^2}{2}W_k+\frac{\gamma_k\tilde{h}L}{2}W_k}{1+\frac{c\sigma(1-\tilde{h})\gamma_k}{\tilde{h}^2-\tilde{h}+1}}\\
			\le & \frac{1}{2}\gamma_k\tilde{h}L\sum_{j=k-\tau_k}^{k-1}W_j+(\frac{\gamma_k\tilde{h}L}{2}-\frac{\tilde{h}^2}{2(1+\frac{c\sigma(1-\tilde{h})\gamma_k}{\tilde{h}^2-\tilde{h}+1})})W_k,
		\end{split}
	\end{equation*}
	i.e., \eqref{eq:asynseq} holds.
	
	Below, we derive \eqref{eq:generalprincipal}. Since $Q_{k+1}\le Q_{t+1}$ $\forall t\in [k-\tau_k, k]$, \eqref{eq:generalprincipal} can be guaranteed by
	\begin{equation}\label{eq:PIAGPLgeneralprincipal}
		\sum_{t=\ell+1}^k p_t\le r_\ell\frac{Q_{k+1}}{Q_{\ell+1}},~ \forall \ell\in [k-\tau_k, k], \quad\forall k\in\mathbb{N}_0.
	\end{equation}
	Note that because $Q_{k+1}\le Q_{\ell+1}$,
	\begin{equation}\label{eq:PIAGrktimesQ}
		\begin{split}
			r_\ell\frac{Q_{k+1}}{Q_{\ell+1}} &= \frac{Q_{k+1}}{Q_{\ell+1}}(\frac{\tilde{h}^2q_\ell}{2}-\frac{\gamma_\ell\tilde{h}L}{2})\\
			&\ge \frac{Q_{k+1}}{Q_{\ell+1}}\frac{\tilde{h}^2q_\ell}{2}-\frac{\gamma_\ell\tilde{h}L}{2}\\
			&\ge (\Pi_{t=\ell}^k q_t)\frac{\tilde{h}^2}{2}-\frac{\gamma_\ell\tilde{h}L}{2}\\
			&\ge (\Pi_{t=k-\tau_k}^k q_t)\frac{\tilde{h}^2}{2}-\frac{\gamma_\ell\tilde{h}L}{2}.
		\end{split}
	\end{equation}
	Because $\sigma\le L$, $\tilde{h} = \frac{1+h}{2}$, $c\le\frac{1-h}{2h}\frac{L}{\sigma}$, and because of \eqref{eq:commonstepsizerule} and $\gamma'=\frac{h}{L}$, we have
	\begin{equation*}
		\sum_{t=k-\tau_k}^k \gamma_t\le \frac{h}{L} \le \frac{\tilde{h}}{L+c\sigma},
	\end{equation*}
	and therefore,
	\begin{equation}\label{eq:PIAGproductofqt}
		\begin{split}
			\Pi_{t=k-\tau_k}^k q_t &= \Pi_{t=k-\tau_k}^k \frac{1}{1+\frac{c\sigma(1-\tilde{h})\gamma_t}{\tilde{h}^2-\tilde{h}+1}}\\
			&\ge \Pi_{t=k-\tau_k}^k(1-\frac{c\sigma(1-\tilde{h})\gamma_t}{\tilde{h}^2-\tilde{h}+1})\\
			&\ge 1-\sum_{t=k-\tau_k}^k \frac{c\sigma(1-\tilde{h})\gamma_t}{\tilde{h}^2-\tilde{h}+1}\\
			&\ge 1-\frac{c\sigma}{L+c\sigma}\frac{\tilde{h}(1-\tilde{h})}{\tilde{h}^2-\tilde{h}+1}\\
			&\ge \frac{L}{L+c\sigma},
		\end{split}
	\end{equation}
	where the last step is due to $\frac{\tilde{h}(1-\tilde{h})}{\tilde{h}^2-\tilde{h}+1}\le 1$ because of $\tilde{h}\in (0,1)$. By \eqref{eq:PIAGrktimesQ}, \eqref{eq:PIAGproductofqt}, \eqref{eq:commonstepsizerule}, $\gamma'=h/L$, and $c\le \frac{1-h}{2h}\frac{L}{\sigma}$, for any $\ell\in [k-\tau_k, k]$,
	\begin{equation*}
		\begin{split}
			\sum_{t=\ell+1}^k p_t-r_\ell\frac{Q_{k+1}}{Q_{\ell+1}}\le& \frac{\tilde{h}L}{2}\sum_{t=\ell}^k \gamma_t - \frac{1}{2}\frac{L\tilde{h}^2}{L+c\sigma}\le 0,
		\end{split}
	\end{equation*}
	i.e., \eqref{eq:PIAGPLgeneralprincipal} holds, which guarantees \eqref{eq:generalprincipal}.

	\section{Proof of Theorem \ref{theo:coor}}\label{proof:theocoor}

	The proof mainly uses Theorem \ref{theo:sequence} and following Lemma, which indicates that some quantities produced by Async-BCD satisfy \eqref{eq:asynseq}.
	\begin{lemma}\label{lemma:ARock}
		Suppose that all the conditions in Theorem \ref{theo:coor} hold. Then, \eqref{eq:asynseq} holds with
		\begin{align*}
			& W_k = 
			\begin{cases}
				\frac{1}{\gamma_k}E[\|x_{k+1}-x_k\|^2], & \gamma_k>0,\\
				0, & \gamma_k=0,
			\end{cases},\\
			& X_{k+1} \!=\!\frac{\gamma_k(1-h)}{4m}E[\|\tilde{\nabla} P(x_k)\|^2],\\
			& V_k\!=\!E[P(x_k)]\!-\!P^\star,~p_k\!=\!\frac{\hat{L}\gamma_k}{2},~q_k=1,~r_k = \frac{h}{2}\!-\!p_k.
		\end{align*}
		In addition, \eqref{eq:generalprincipal} holds.
	\end{lemma}
	By Theorem \ref{theo:sequence} and Lemma \ref{lemma:ARock}, Theorem \ref{theo:coor} is straightforward.
	
	\subsection{Proof of Lemma \ref{lemma:ARock}}
	
	If $\gamma_k = 0$, then $p_k=0$. Below we consider the case where $\gamma_k>0$ and prove \eqref{eq:asynseq} and \eqref{eq:generalprincipal}.
	
	Suppose that at the $k$th iteration, the $j'$-th block is updated. Define $u_k = \frac{x_k^{(j')}-x_{k+1}^{(j')}}{\gamma_k}-\nabla_{j'} f(\hat{x}_k)$. By the first-order optimality condition of \eqref{eq:asyncor},
	\begin{equation*}
		u_k\in \partial R^{(j')}(x_{k+1}^{(j')}),
	\end{equation*}
	which, together with the convexity of $R^{(j')}$, yields
	\begin{equation*}
		\begin{split}
			R(x_{k+1}) - R(x_k) &= R^{(j')}(x_{k+1}^{(j')})-R^{(j')}(x_k^{(j')})\\
			&\le \langle u_k, x_{k+1}^{(j')}-x_k^{(j')}\rangle.
		\end{split}
	\end{equation*}
	By the Lipschitz continuity in Assumption \ref{asm:sc},
	\begin{equation*}
		\begin{split}
			f(x_{k+1}) - f(x_k)\le & \langle \nabla_{j'}f(x_k),x_{k+1}^{(j')}-x_k^{(j')}\rangle\!+\!\frac{\hat{L}}{2}\|x_{k+1}^{(j')}\!-\!x_k^{(j')}\|^2.
		\end{split}
	\end{equation*}
	Adding two equations above gives
	\begin{equation}\label{eq:cor1}
		\begin{split}
			P(x_{k+1})-P(x_k)\le& \langle \nabla_{j'} f(x_k)\!+\!u_k, x_{k+1}^{(j')}\!-\!x_k^{(j')}\rangle\!\!+\!\!\frac{\hat{L}}{2}\|x_{k+1}^{(j')}\!-\!x_k^{(j')}\|^2.
		\end{split}
	\end{equation}
	In the above equation,
	\begin{equation}\label{eq:cor2}
		\begin{split}
			&\langle \nabla_{j'} f(x_k)+u_k, x_{k+1}^{(j')}-x_k^{(j')}\rangle\\
			=& \frac{\gamma_k}{2}\|\nabla_{j'} f(x_k)+u_k+\frac{x_{k+1}^{(j')}-x_k^{(j')}}{\gamma_k}\|^2-\frac{\gamma_k}{2}\|\nabla_{j'} f(x_k)\!+\!u_k\|^2\!-\!\frac{1}{2\gamma_k}\|x_{k+1}^{(j')}\!-\!x_k^{(j')}\|^2.
		\end{split}
	\end{equation}
	By the definition of $u_k$,
	\begin{equation}\label{eq:cor3}
		\begin{split}
			&   \|\nabla_{j'} f(x_k)+u_k+\frac{x_{k+1}^{({j'})}-x_k^{({j'})}}{\gamma_k}\|^2= \|\nabla_{j'} f(x_k)-\nabla_{j'} f(\hat{x}_k)\|^2.
		\end{split}
	\end{equation}
	Substituting \eqref{eq:cor2} and \eqref{eq:cor3} into \eqref{eq:cor1} gives
	\begin{equation}\label{eq:pxk1xkdiff}
		\begin{split}
			&P(x_{k+1})-P(x_k)\\
			\le&\frac{\gamma_k}{2}\|\nabla_{j'} f(x_k)\!-\!\nabla_{j'} f(\hat{x}_k)\|^2-\!\frac{\gamma_k}{2}\|\nabla_{j'} f(x_k)\!+\!u_k\|^2-\frac{1/\gamma_k-\hat{L}}{2}\|x_{k+1}^{(j')}\!-\!x_k^{(j')}\|^2.
		\end{split}
	\end{equation}
	In addition,
	\begin{equation}\label{eq:gammaknormsquare}
		\begin{split}
			&\gamma_k\|\frac{\operatorname{prox}_{\gamma_k R^{(j')}}(x_k^{(j')}-\gamma_k \nabla_{j'} f(x_k))-x_k^{(j')}}{\gamma_k}\|^2\\
			\le & \frac{2}{\gamma_k}\|\operatorname{prox}_{\gamma_k R^{(j')}}(x_k^{(j')}-\gamma_k \nabla_{j'} f(x_k))-x_{k+1}^{(j')}\|^2\\
			&+\frac{2}{\gamma_k}\|x_{k+1}^{(j')}-x_k^{(j')}\|^2\\
			\le &2\gamma_k\|\nabla_{j'} f(\hat{x}_k)\!-\!\nabla_{j'} f(x_k)\|^2\!+\!\frac{2}{\gamma_k}\|x_{k+1}^{(j')}\!-\!x_k^{(j')}\|^2,
		\end{split}
	\end{equation}
	where the last inequality comes from the non-expansive property of the proximal operator. Multiplying both sides of \eqref{eq:gammaknormsquare} by $\frac{1-h}{4}$ and adding the resulting equation with \eqref{eq:pxk1xkdiff}, we derive
	\begin{equation}\label{eq:cor4}
		\begin{split}
			&P(x_{k+1})-P(x_k)\le-\frac{\gamma_k(1-h)}{4}\\
			&\cdot\|\frac{\operatorname{prox}_{\gamma_k R^{(j)}}(x_k^{(j')}-\gamma_k \nabla_{j'} f(x_k))-x_k^{(j')}}{\gamma_k}\|^2\\ &+\frac{\gamma_k(2-h)}{2}\|\nabla_{j'} f(x_k)-\nabla_{j'} f(\hat{x}_k)\|^2\\
			&-\frac{h/\gamma_k - \hat{L}}{2}\|x_{k+1}^{(j')}\!-\!x_k^{(j')}\|^2.
		\end{split}
	\end{equation}
	Moreover, by \eqref{eq:inconsistentreadhatx}, $J_k\subseteq [k-\tau_k,k]$, Assumption \ref{asm:sc}, the step-size condition \eqref{eq:commonstepsizerule}, $\gamma'=\frac{h}{\hat{L}}$, and the Cauchy-Schwartz inequality,
	\begin{equation}\label{eq:gradiffcor}
		\begin{split}
			& \|\nabla_{j'} f(x_k)-\nabla_{j'} f(\hat{x}_k)\|^2\\
			=& \|\sum_{t\in J_k} \nabla_{j'} f(x_{t+1})-\nabla_{j'} f(x_t)\|^2\\
			\le& \|\sum_{t=k-\tau_k}^{k-1} \nabla_{j'} f(x_{t+1})-\nabla_{j'} f(x_t)\|^2\\
			\le& (\sum_{t=k-\tau_k}^{k-1}\|\nabla_{j'} f(x_{t+1})-\nabla_{j'} f(x_t)\|)^2\\
			=& (\sum_{t=k-\tau_k}^{k-1}\sqrt{\gamma_t}\frac{\|\nabla_{j'} f(x_{t+1})-\nabla_{j'} f(x_t)\|}{\sqrt{\gamma_t}})^2\\
			\le& (\sum_{t=k-\tau_k}^{k-1} \gamma_t)\sum_{t=k-\tau_k}^{k-1}\frac{\|\nabla_{j'} f(x_{t+1})-\nabla_{j'} f(x_t)\|^2}{\gamma_t}\\
			\le& \hat{L}h\sum_{t=k-\tau_k}^{k-1}\frac{\|x_{t+1}-x_t\|^2}{\gamma_t},
		\end{split}
	\end{equation}
	and according to Lemma 1 in \cite{karimi16}, because $\gamma_k\le \frac{1}{\hat{L}}$,
	\begin{equation}\label{eq:proximaldistance}
		\|\frac{\operatorname{prox}_{\gamma_k R}(x_k-\gamma_k \nabla f(x_k))-x_k}{\gamma_k}\|\ge \|\tilde{\nabla} P(x_k)\|.
	\end{equation}
	Substituting \eqref{eq:gradiffcor} and \eqref{eq:proximaldistance} into \eqref{eq:cor4} and using $(2-h)h\le 1$, we have
	\begin{equation*}
		\begin{split}
			&E[P(x_{k+1})|x_k]-P(x_k)\\
			\le&-\frac{\gamma_k(1-h)}{4m}\|\tilde{\nabla} P(x_k)\|^2+\frac{\hat{L}\gamma_k}{2}\sum_{t=k-\tau_k}^{k-1} W_t-\frac{h-\hat{L}\gamma_k}{2}W_k.
		\end{split}
	\end{equation*}
	Therefore, \eqref{eq:asynseq} holds.
	\section{Proof of Proposition \ref{prop:stepsizeintegralbound}}\label{sec:proofprop}
	
	\textbf{Proof of \eqref{eq:stepsizesummationPIAG}}:	To derive \eqref{eq:stepsizesummationPIAG} for adaptive step-size  \eqref{eq:PIAGadapt}, define $\{t_k\}_{k=0}^\infty$ as $t_0=0$,  $t_{k+1} = \min\{t:~t-\tau_t>t_k\}$ $\forall k\in\mathbb{N}_0$ and $N_k=1+\max\{j: t_j\le k\}$ $\forall k\in\mathbb{N}_0$. Because $\tau_j\le \tau$ $\forall j\in\N_0$, by the definition of $t_j$,
	\begin{equation*}
		t_{j+1}\le t_j+\tau+1,~\forall j\in\N_0.
	\end{equation*}
	In addition, $t_0=0$. Then, we have $t_j\le j(\tau+1)$, which implies $t_{\lfloor k/(\tau+1) \rfloor}\le k$. By definition of $N_k$,
	\begin{equation*}
		\begin{split}
			N_k&\ge 1+\max\{j: t_j\le k\}\\
			&\ge 1+\lfloor k/(\tau+1)\rfloor \ge \frac{k+1}{\tau+1}.
		\end{split}
	\end{equation*}
	Note that because $t_k+1\le t_{k+1}-\tau_{t_{k+1}}$,
	\begin{equation*}
		\sum_{j=t_k+1}^{t_{k+1}}\gamma_j\ge \sum_{j=t_{k+1}-\tau_{t_{k+1}}}^{t_{k+1}} \gamma_j.
	\end{equation*}
	In addition, by the definition of $N_k$, $t_{N_k-1}\le k$, which, together with the above equation, gives
	\begin{equation}\label{eq:stepsizeintegralinproof}
		\begin{split}
			\sum_{j=0}^k \gamma_j &\ge \sum_{j=0}^{t_{N_k-1}} \gamma_j = \gamma_0+\sum_{\ell=0}^{N_k-2} \sum_{j=t_{\ell}+1}^{t_{\ell+1}}\gamma_j\\
			&\ge\gamma_0+\sum_{\ell=0}^{N_k-2} \sum_{j=t_{\ell+1}-\tau_{t_{\ell+1}}}^{t_{\ell+1}} \gamma_j.				
		\end{split}
	\end{equation}
	If $\gamma_{t_{\ell+1}}=0$, then $a_{t_{\ell+1}}<\gamma' - \sum_{j=t_{\ell+1}-\tau_{t_{\ell+1}}}^{t_{\ell+1}-1} \gamma_j$, which implies
	\begin{equation*}
		\sum_{j=t_{\ell+1}-\tau_{t_{\ell+1}}}^{t_{\ell+1}-1} \gamma_j \ge \gamma'.
	\end{equation*}
	Otherwise, $\gamma_{t_{\ell+1}} = \alpha (\gamma'-\sum_{j=t_{\ell+1}-\tau_{t_{\ell+1}}}^{t_{\ell+1}-1} \gamma_j)$ and
	\begin{equation*}
		\sum_{j=t_{\ell+1}-\tau_{t_{\ell+1}}}^{t_{\ell+1}} \gamma_j = \alpha\gamma'+(1-\alpha)\sum_{j=t_{\ell+1}-\tau_{t_{\ell+1}}}^{t_{\ell+1}-1}\ge \alpha\gamma'.
	\end{equation*}
	From the above two equations, we have $\sum_{j=t_{\ell+1}-\tau_{\ell+1}}^{t_{\ell+1}} \gamma_j\ge \alpha\gamma'$. In addition, because $\tau_k\in [0,k]$ $\forall k\in\N_0$, $\tau_0=0$ and $\gamma_0=\alpha\gamma'$. Substituting these into \eqref{eq:stepsizeintegralinproof} yields \eqref{eq:stepsizesummationPIAG}.
	
	\textbf{Proof of \eqref{eq:stepsizeintegral2}}: We use mathematical induction to prove \eqref{eq:stepsizeintegral2} for adaptive step-size \eqref{eq:adapt2}. Suppose that the following equation holds at some $k\in \N_0$:
	\begin{equation}\label{eq:induction}
		\sum_{j=0}^\ell \gamma_j \ge \frac{\tau}{\tau+1}\cdot\frac{\gamma'(\ell+1)}{\tau+1}, \forall \ell\le k-1,
	\end{equation}
	which naturally holds when $k=0$. Below, we prove that \eqref{eq:induction} holds at $k+1$ by showing
	\begin{equation}\label{eq:inductionatkplus1}
		\sum_{j=0}^k \gamma_j \ge \frac{\tau}{\tau+1}\cdot\frac{\gamma'(k+1)}{\tau+1}.
	\end{equation}

	If $\gamma_k=0$, which is possible only when $\tau_k>0$, then by \eqref{eq:adapt2},
	\begin{equation*}
		\sum_{j=k-\tau_k}^{k-1} \gamma_j>\frac{\tau_k \gamma'}{\tau_k+1}\ge \frac{\tau}{\tau+1}\cdot\frac{(\tau_k+1) \gamma'}{\tau+1},
	\end{equation*}
	where the last step is due to $\frac{\tau_k}{(\tau_k+1)^2}\ge \frac{\tau}{(\tau+1)^2}$ when $\tau_k\in [1,\tau]$. In addition, because $k-\tau_k-1\le k-1$, by \eqref{eq:induction},
	\begin{equation*}
		\sum_{j=0}^{k-\tau_k-1} \gamma_j \ge \frac{\tau}{\tau+1}\cdot\frac{\gamma'(k-\tau_k)}{\tau+1}.
	\end{equation*}
	Adding the two equations above and using $\gamma_k=0$, we have \eqref{eq:inductionatkplus1}. If $\gamma_k>0$, then by \eqref{eq:adapt2}, $\gamma_k\ge \frac{\gamma'}{\tau_k+1}\ge \frac{\gamma'}{\tau+1}$. In addition, $\sum_{j=0}^{k-1} \gamma_j \ge \frac{\gamma'k}{\tau+1}$ by \eqref{eq:induction}. Then, we have \eqref{eq:inductionatkplus1}, which indicates \eqref{eq:induction} at $k+1$. Conclude all the above, \eqref{eq:induction} as well as \eqref{eq:stepsizeintegral2} holds for all $k\in \N_0$.

\bibliographystyle{unsrt}
\bibliography{references}

\end{document}